\documentclass[letterpaper]{article} 
\usepackage{aaai24}  
\usepackage{times}  
\usepackage{helvet}  
\usepackage{courier}  
\usepackage[hyphens]{url}  
\usepackage{graphicx} 
\def\UrlFont{\rm}  
\usepackage{natbib}  
\usepackage{caption} 
\usepackage{algorithm}
\usepackage{algorithmic}

\usepackage{newfloat}
\usepackage{listings}
\DeclareCaptionStyle{ruled}{labelfont=normalfont,labelsep=colon,strut=off} 
\floatstyle{ruled}
\newfloat{listing}{tb}{lst}{}
\floatname{listing}{Listing}
\title{Generalisation Through Negation and Predicate Invention}

\author {
   David M. Cerna\textsuperscript{\rm 1}, Andrew Cropper\textsuperscript{\rm 2}
}
\affiliations{
    \textsuperscript{\rm 1}Czech Academy of Sciences Institute of Computer Science (CAS ICS), Prague, Czechia\\
    \textsuperscript{\rm 2}University of Oxford,  United Kingdom\\
dcerna@cs.cas.cz, andrew.cropper@cs.ox.ac.uk%
}

\usepackage{bbding}
\usepackage{xcolor}
\usepackage{booktabs}
\usepackage{amsthm}
\usepackage{inconsolata}
\usepackage{tikz}
 \usepackage{amsmath}
\usepackage{microtype}
\usepackage{xfp}
\usepackage{amssymb}
\usepackage[titletoc]{appendix}
\usepackage{bm}
\usepackage{stackengine}

\newcommand{\name}{\textsc{Nopi}}

\newcommand{\popper}{\textsc{Popper}}

\newcommand{\ale}{\textsc{Aleph}}

\newcommand{\metagolsn}{\textsc{Metagol$_{SN}$}}
\newcommand{\nameb}{\textsc{Nopi$_{bn}$}}

\newcommand{\parOp}{\ensuremath{\bowtie}}

\newcommand{\symc}[1]{\ensuremath{head_s(#1)}}

\newcommand{\symT}[1]{\ensuremath{\mathit{sym}_T(#1)}}
\newcommand{\even}{\ensuremath{\bm{+}}}
\newcommand{\odd}{\ensuremath{\bm{-}}}

\newcommand{\body}[1]{\ensuremath{\mathit{body}(#1)}}
\newcommand{\bodyp}[1]{\ensuremath{\mathit{body}^{\even}(#1)}}
\newcommand{\bodyn}[1]{\ensuremath{\mathit{body}^{\odd}(#1)}}
\newcommand{\bodyv}[2]{\ensuremath{\mathit{body}^{#2}(#1)}}

\newcommand{\dneg}{\ensuremath{\mathbf{not}}}
\newcommand{\psub}{\ensuremath{\preccurlyeq_{\diamond}}}
\newcommand{\sub}{\ensuremath{\preceq_{\theta}}}

\newtheorem{definition}{Definition}
\newtheorem{example}{Example}
\newtheorem{theorem}{Theorem}
\newtheorem{lemma}{lemma}
\newtheorem{proposition}{Proposition}

\usepackage{pgfkeys}
    \def\addlegendimage{\scriptsize\csname pgfplots@addlegendimage\endcsname}
\usepackage{comment}

\lstnewenvironment{myalgorithm}[1][] 
{
    \lstset{ 
        showspaces=false,               
        showstringspaces=false,         
        mathescape=true,
        numbers=left,
        escapeinside={*}{*},
        columns=flexible,
        keywordstyle=\bfseries,
        keywords={,and, return, not, def, in, if, else, for, foreach, while, }
        numbers=left,
        #1 
    }
}
{}

\usepackage{bibentry}

\begin{document}

\maketitle
\begin{abstract}
The ability to generalise from a small number of examples is a fundamental challenge in machine learning. 
To tackle this challenge, we introduce an inductive logic programming (ILP) approach that combines negation and predicate invention. 
Combining these two features allows an ILP system to generalise better by learning rules with universally quantified body-only variables.
We implement our idea in \name{}, which can learn normal logic programs with predicate invention, including Datalog programs with stratified negation.
Our experimental results on multiple domains show that our approach can improve predictive accuracies and learning times.
\end{abstract}
\section{Introduction}
Zendo is a game where one player, the \emph{teacher}, creates a hidden rule for structures.
The other players, the \emph{students}, aim to discover the rule by building structures.
The teacher provides feedback by marking which structures follow or break the rule without further explanation.
The students continue to guess the rule. 
The first student to correctly guess the rule wins.
For instance, consider the examples shown in Figure~\ref{fig:motivatingExample}.
A possible rule for these examples is \emph{``there are two red cones''}.

\begin{figure}[ht]
\centering
\newcommand\mleft{.19}
\newcommand\mright{.15}
\newcommand\wleft{\fpeval{\mright/\mleft}}
\begin{minipage}{.5\textwidth}
\includegraphics[width=.28\textwidth]{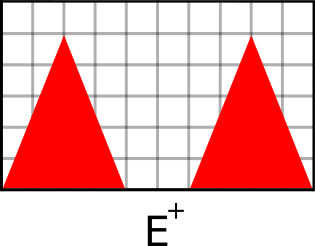}
\hspace{.4em}
\includegraphics[width=.28\textwidth]{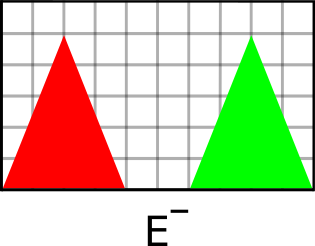}
\hspace{.4em}
\includegraphics[width=.28\textwidth]{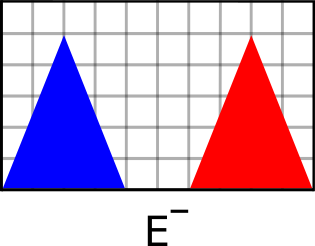}
\end{minipage}
    \caption{
    Positive (E$^+$ ) and negative (E$^-$) Zendo examples.    
    }
    \label{fig:motivatingExample}
\end{figure}

\noindent
Suppose we want to use machine learning to play Zendo, i.e. to learn rules from examples.
Then we need an approach that can (i) learn explainable rules, and (ii) generalise from a small number of examples. 
Although crucial for many problems, these requirements are difficult for standard machine learning techniques \cite{ilp30}.

Inductive logic programming (ILP) \cite{mugg:ilp} is a form of machine learning that can learn explainable rules from a small number of examples. 
For instance, an ILP system could learn the following hypothesis (a set of logical rules) from the examples in Figure \ref{fig:motivatingExample}:
\[
    \begin{array}{l}
    \left\{
    \begin{array}{l}
\emph{f(S) $\leftarrow$ cone(S,A),red(A),cone(S,B),red(B),all\_diff(A,B)}\\
    \end{array}
    \right\}
    \end{array}
\]

\noindent
This hypothesis says that the relation \emph{f} holds for the state $S$ when there are two distinct red cones $A$ and $B$,  i.e. this hypothesis says there are two red cones.

\begin{figure}[ht]
\centering
\includegraphics[width=.18\textwidth]{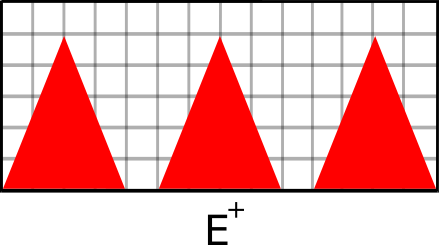}
\hspace{.4em}
\includegraphics[width=.17\textwidth]{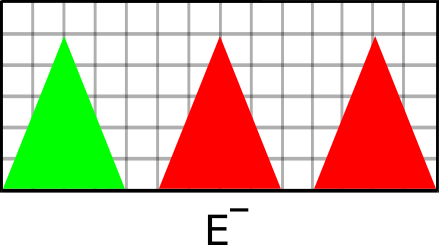}
    \caption{Additional Zendo examples.}
    \label{fig:motivatingExample2}
\end{figure}

\noindent
Suppose we are given the two new examples shown in Figure \ref{fig:motivatingExample2}.
Our previous hypothesis does not correctly explain the new examples as it entails the new negative example.
To correctly explain all the examples, we need a disjunctive hypothesis that says \emph{``[there are exactly two red cones] or [there are exactly three red cones]''}.
Given a new positive example with four red cones and a negative example with three red and one green cone, we would need to learn yet another rule that says \emph{``there are exactly four red cones''}.

As is hopefully clear, we would struggle to generalise beyond the training examples using this approach because we need to learn a rule for each number of cones.
Rather than learn a rule for each number of cones, we would ideally learn a single rule that says \emph{``all the cones are red''}.
However, most ILP approaches struggle to learn rules of this form because they only learn Datalog or definite programs and thus only learn rules with \emph{existentially} quantified body-only variables ~\cite{AptB91,datsin:datalog}.

To overcome this limitation, we combine \emph{negation as failure} (NAF) \cite{naf} and \emph{predicate invention} (PI) \cite{stahl:pi} to learn rules with \emph{universally} quantified body-only variables.
The main reason to combine NAF and PI is that many concepts can only be expressed in this more expressive language ~\cite{stahl:pi,datsin:datalog}.
For instance, for the Zendo scenario, our approach,
which combines negation and PI, learns the hypothesis:
\[
    \begin{array}{l}
    \left\{
    \begin{array}{l}
    \emph{f(S) $\leftarrow$ scene(S), not inv$_1$(S)}\\
    \emph{inv$_1$(S) $\leftarrow$ cone(S,P), not red(P)}\\
    \end{array}
    \right\}
    \end{array}
\]

\noindent
This hypothesis says \emph{``all the cones are red''}.
The predicate symbol \emph{inv$_1$} is not provided as input and is invented by our approach. 
The rule defined by \emph{inv$_1$} says \emph{``there is a cone that is not red''}.
The rule defined by \emph{f} negates this rule and says \emph{``it is not true that there is a cone that is not red''}. 
The hypothesis, therefore, states  \emph{``there does not exist a cone that is not red''} which by the equivalences of first-order logic ($\forall \equiv \textbf{not}\ \exists\ \textbf{not}$) is the same as \emph{``all the cones are red''}.



To combine negation and PI, we build on \emph{learning from failures} (LFF) \cite{popper}.
An LFF learner continually generates and tests hypotheses, from which it infers constraints. 
For instance, if a hypothesis is too general, i.e. entails a negative example, an LFF learner, such as \popper{}, builds a generalisation constraint to prune more general hypotheses from the hypothesis space.
We extend LFF from learning definite (monotonic) programs to learning \textit{polar} programs, a fragment of normal (non-monotonic) programs.
The key benefit of polar programs is that we can efficiently reason about the subsumption relation between them in our learning algorithm.
Furthermore, we show (Theorem \ref{thm:stratified}) that polar programs capture Datalog with stratified negation \cite{datsin:datalog}.
We implement our idea in \name{}, which, as it builds on \popper{}, supports learning recursive and optimal programs.

\paragraph{Novelty, impact, and contributions.}
The key novelty of our approach is \emph{the ability to learn normal logic programs with invented predicate symbols}.
We expand on this novelty in Section \ref{sec:related}.
The impact is that our approach can learn programs that existing approaches cannot.
Specifically, we claim that combining negation and PI can improve learning performance by allowing us to learn rules with universally quantified body-only variables.
Our experiments on multiple domains support our claim and show that our approach leads to vastly improved predictive accuracies and learning times.

Overall, we make the following contributions:
\begin{enumerate}
\item We introduce \textit{polar} programs, a fragment of \textit{stratified} logic programs.
We show (Theorem \ref{thm:stratified}) that this fragment of normal logic programs can capture Datalog with stratified negation \cite{datsin:datalog}.

\item We introduce LFF constraints for this non-monotonic setting and prove their soundness (Propositions \ref{prop:gen_sound} and \ref{prop:spec_sound}).

\item We introduce \name{}, an ILP system that can learn normal logic programs with PI and recursion, such as Datalog programs with stratified negation.

\item We empirically show on multiple domains that (i) \name{} can outperform existing approaches and (ii) our non-monotonic constraints can reduce learning times.
\end{enumerate}

\section{Related Work}
\label{sec:related}

\textbf{Program synthesis.}
ILP is a form of program synthesis \cite{mis}, which attracts a broad community of researchers \cite{dilp,ellis:scc,tom_pi}.
Many recent approaches synthesise monotonic Datalog programs \cite{difflog,prosynth,shitruleselection}. 
We differ in many ways, including by learning non-monotonic programs.

\textbf{Negation.}
Many ILP approaches learn non-monotonic programs \cite{foil,srin:noise,DBLP:conf/ecml/DimopoulosK95,DBLP:conf/lpnmr/Sakama01,DBLP:journals/ml/SakamaI09,xhail}.
Most use negation to handle exceptions such as \emph{``birds fly except penguins''} and thus require negative examples.
For instance, \citet{DBLP:conf/ijcai/InoueK97} learn normal logic programs by first learning a program that covers the positive examples and then adding exceptions (using NAF) to account for the negative examples. 
By contrast, we combine NAF and PI to improve generalisation and do not need negative examples --
as \citet{BekkerD20} state, it is sometimes necessary to learn from  positive examples alone. 
Moreover, most approaches build on inverse entailment \cite{progol}, so they struggle to learn recursive and optimal programs.
By contrast, our approach can learn recursive and optimal programs because we build on LFF.
As \citet{FogelZ98} state, learning non-monotonic programs is difficult because the standard subsumption relation does not hold in general for normal programs.
To overcome this challenge, the authors introduce a subsumption relation for normal programs based on the dependency graph of predicate symbols in a program. 
We differ because we introduce a general fragment of normal logic programs related to stratified logic programs.
Moreover, our approach supports PI.

\textbf{Predicate invention.}
Although crucial for many tasks, such as planning \cite{tom_pi} and learning complex algorithms \cite{metaopt}, most ILP systems do not support PI \cite{progol,aleph,tilde,aspal,quickfoil,inoue:lfit}.
Approaches that support PI usually need metarules to restrict the syntax of hypotheses \cite{mugg:metagold,dilp,hexmil,celine:bottom,dai2020abductive,DBLP:conf/icml/GlanoisJFWZ0LH22}, which, in some cases, are impossible to provide \cite{reduce}.
By contrast, we do not need metarules.

\textbf{Negation and predicate invention.}
\citet{DBLP:journals/jiis/Ferilli16a} describe an approach that specialises a theory to account for a misclassified negative example.
If a negative example is misclassified, they introduce a conjunction of negated preconditions, where each precondition is an invented predicate.
Their approach only works in a Datalog setting, cannot learn recursive programs, and only works when a negative example is misclassified.
We differ because we (i) do not need negative examples, (ii) learn recursive programs, (iii) learn normal logic programs, and (iv) learn optimal programs. 
\citet{stepwise2018} learn recursive programs with negation and PI.
Their approach first learns a program for the positive examples and allows some negative examples to be covered.
It then flips the examples (false positives from the previous iteration are positive examples, and the previous true positives are now negative examples) and tries to learn again.
We differ because we do not need negative examples or metarules.
ILASP~\cite{ilasp} can learn non-monotonic programs with invented predicate symbols if a user tells it which symbols to invent.
By contrast, \name{} does not need this information. 
Moreover, ILASP precomputes every possible rule in the hypothesis 
space, which is often infeasible. For instance, our Zendo4 experiment has approximately $10^{10}$ rules in the hypothesis space.

\section{Problem Setting}
\label{sec:setting}
We assume familiarity with logic programming \cite{lloyd:book} and ASP \cite{asp} but have included summaries in Appendix~\ref{sec:terminology}. 
For clarity, we define some key terms.
A \textit{normal} rule is of the form
$
    h\leftarrow b_1,\cdots, b_n, \dneg\ b_{n+1}, \cdots, \dneg \ b_{n+m}
$
where $h$ is the \emph{head} atom, each $b_i$ is a literal, and $b_1,\cdots, b_n, \dneg\ b_{n+1}, \cdots, \dneg \ b_{n+m}$ is the \emph{body}.
The symbol $\dneg$ denotes \textit{negation as failure}~\cite{naf}. 
A literal is an atom (a non-negated literal) or the negation of an atom (a negated literal).
A \emph{normal} logic program is a set of normal rules.
A clause is a set of literals.
A \emph{definite} clause is a clause with exactly one non-negated literal. 
A \emph{substitution} $\theta = \{v_1 / t_1, ..., v_n/t_n \}$ is the simultaneous replacement of each variable $v_i$ by its corresponding term $t_i$. 
A clause $C_1$ \emph{subsumes} a clause $C_2$ if and only if there exists a substitution $\theta$ such that $C_1 \theta \subseteq C_2$ \cite{plotkin:thesis}.
A definite theory $P$ subsumes a definite theory $Q$ ($P \sub Q$) if and only if $\forall r_2 \in Q, \exists r_1 \in P$ such that $r_1$ subsumes $r_2$.
A definite theory $P$ is a \emph{specialisation} of a definite theory $Q$ if and only if $Q \sub P$.
A definite theory $P$ is a \emph{generalisation} of a definite theory $Q$ if and only if $P \sub Q$.

\subsection{Polar Programs}
To learn normal programs, we need to go beyond definite programs and standard subsumption.
To do so, we introduce \emph{polar programs}, which are normal programs where predicate symbols have polarities.
We first define \emph{top symbols}, which are head predicate symbols that are only used positively in a program:

\begin{definition}[\textbf{Top symbols}]
\label{def:topset}
  Let $P$ be a normal program. Then $top(P)$ is the inclusion-maximal subset of the head predicate symbols occurring in $P$ satisfying the follow two conditions both hold:
  \begin{itemize}
      \item if $p\in top(P)$, then $p$ does not occur in a negated literal in $P$
      \item if $p\in top(P)$ and $p$ is in the body of a rule $r$ then the head predicate symbol of $r$ is in $top(P)$
  \end{itemize}
\end{definition}
\noindent
We define \emph{defs(P)} as the set of all head predicate symbols in $P$ that are not in \emph{top(P)}.



\begin{example}
\label{ex:top}
To illustrate top symbols, consider the program: 
\[  
 \begin{array}{l}
     P =
    \left\{
    \begin{array}{l}
    p\leftarrow g, q \\
    p\leftarrow w, \mathbf{not}~s \\
    g\leftarrow q, p \\
    s\leftarrow l
    \end{array}
    \right\}
    \end{array}
\]
\noindent
In this program, $top(P)= \{ p,g\}$ and $defs(P)=\{s\}$. 
\end{example}

\noindent 
In a normal program $P$, the polarity of every head predicate symbol $p$ is \emph{positive} $(pos(p))$ or \emph{negative} $(neg(p))$.
The polarity of a symbol in \emph{top(P)} is positive. By contrast, the polarity of a symbol in \emph{defs(P)} depends on whether the symbol is used positively or negatively:

\begin{definition}[\textbf{Polarity}]
\label{def:polarity}
Let $P$ be a normal program, $r$ be a rule in $P$, 
$p$ be the head predicate symbol of $r$, 
$\bodyp{r}$ and $\bodyn{r}$ be the predicate symbols that appear in non-negated and negated body literals in $r$ respectively, and $q$ in $defs(P)$ be a predicate symbol in the body of $r$.
Then the polarity of $q$ is as follows:
\begin{itemize}
\item[(1)] if $q\in \bodyp{r}$ and $pos(p)$ then $pos(q)$
\item[(2)]  if $q\in \bodyn{r}$ and $pos(p)$ then $neg(q)$
\item[(3)]  if $q\in \bodyp{r}$ and $neg(p)$ then $neg(q)$ 
\item[(4)]  if $q\in \bodyn{r}$ and $neg(p)$ then $pos(q)$
\end{itemize} 
\end{definition}
\noindent

\begin{example}
Consider the program $P$ from Example \ref{ex:top}.
The polarities of the head predicate symbols are $pos(p)$, $pos(g)$, and $neg(s)$.
\end{example}

\noindent
We define a \emph{polar program}:
\begin{definition}[\textbf{Polar program}]
\label{def:polarProg}
A normal program $P$ is polar if and only if the polarity of every head predicate symbol in $P$ is exclusively positive or negative.
\end{definition}

\begin{example}
The following program is not polar because the polarity of \emph{odd} is neither positive nor negative:

\[
  \begin{array}{l}
    \left\{
        \emph{ ${odd}$(X) $\leftarrow$  succ(Y,X), \textbf{not}\ ${odd}$(Y)}
    \right\}
    \end{array}
\]

\noindent
\end{example}
\begin{example}
The following \textit{stratified} program is not polar because the polarity of $inv_1$ is positive and negative:
\[
 \begin{array}{l}
    \left\{
    \begin{array}{lll}
    \emph{f $\leftarrow$ inv$_1$, \textbf{not}\ inv$_1$}\\    \emph{inv$_1$ $\leftarrow$ inv$_2$} \\
    \emph{inv$_2$ $\leftarrow$  w}
    \end{array}
    \right\}
    \end{array}
\]
\end{example}
\begin{example}
\label{ex:rec}
The following program is \textit{polar} because only $pos(unconnected)$ and $neg(inv_1)$ hold:
\[
 \begin{array}{l}
    \left\{
    \begin{array}{l}
    \emph{$r_1:$ ${unconnected}(A,B)$ $\leftarrow$ \textbf{not}\ $inv_1$(A,B)}\\
    \emph{$r_2:$ ${inv}_1$(A,B) $\leftarrow$ ${edge}$(A,B)}\\
    \emph{$r_3:$ ${inv}_1$(A,B) $\leftarrow$ ${edge}$(A,C), ${inv}_1$(C,B)}
    \end{array}
    \right\}
    \end{array}
\]
\end{example}
\noindent
The rules of a polar program $P$ are \textit{positive} ($P^{\even}$) or \textit{negative} ($P^{\odd}$) depending on the polarity of their head symbols.

\begin{example}
    Consider the program P from Example~\ref{ex:rec}. 
Then $P^{\even}=\{r_1\}$ and $P^{\odd}=\{r_2,r_3\}$.
\end{example}
\noindent
We can compare positive rules using standard subsumption. 
For negative rules, we need to flip the order of comparison.
To do so, we introduce \emph{polar subsumption}:
\begin{definition}[\textbf{Polar subsumption}]
\label{def:polarsub}
Let $P$ and $Q$ be polar programs.
Then $P$ \textit{polar subsumes} $Q$ ($P\psub Q$)  
iff $P^{\even}\preceq_{\theta} Q^{\even}$ and $Q^{\odd}\preceq_{\theta}P^{\odd}$.
\end{definition}
\noindent 
\noindent
\begin{example}
To understand the intuition behind Definition~\ref{def:polarsub}, consider the following 
polar programs:        
\begin{minipage}{.24\textwidth}
\[
 P = \hspace{-.8em}\begin{array}{l}
    \left\{
    \begin{array}{l}
    \emph{$r_1$:\ f $\leftarrow$ \textbf{not}  $inv_1$ }\\
    \emph{$r_2$:\  $inv_1$  $\leftarrow$ a, b}
    \end{array}
    \right\}
    \end{array}
\]
\end{minipage}
\begin{minipage}{.17\textwidth}
\[
 Q =  \hspace{-.8em} \begin{array}{l}
    \left\{
    \begin{array}{l}
    \emph{$r_3$:\ f $\leftarrow$ \textbf{not}  $inv_1$ }\\
    \emph{$r_4$:\ $inv_1$  $\leftarrow$ a}
    \end{array}
    \right\}
    \end{array}
\]
\end{minipage}

\noindent
Note that $P\psub Q$ because $r_1 \sub r_3$ where $r_1\in P^{\even}$ and $r_3\in Q^{\even}$, and  $r_4\sub r_2$ where $r_2\in P^{\odd}$ and $r_4\in Q^{\odd}$ . 

\noindent
\begin{minipage}{.24\textwidth}
\[
 P = \hspace{-.8em}   \begin{array}{l}
    \left\{
    \begin{array}{l}
    \emph{$r_1$:\ f $\leftarrow$ \textbf{not}  $inv_1$ }\\
    \emph{$r_2$:\  $inv_1$  $\leftarrow$ b}\\
    \emph{$r_3$:\  $inv_1$  $\leftarrow$ a}
    \end{array}
    \right\}
    \end{array}
\]
\end{minipage}
\begin{minipage}{.17\textwidth}
\[
 Q =  \hspace{-.8em}  \begin{array}{l}
    \left\{
    \begin{array}{l}
    \emph{$r_4$:\ f $\leftarrow$ \textbf{not}  $inv_1$ }\\
    \emph{$r_5$:\  $inv_1$  $\leftarrow$ b}\\
    \end{array}
    \right\}
    \end{array}
\]
\end{minipage}

\noindent Note that $Q\psub P$ because $r_4 \sub r_1$ where $r_1\in P^{\even}$ and $r_4\in Q^{\even}$, and $r_2\sub r_5$ where $r_2\in P^{\odd}$ and $r_5\in Q^{\odd}$. 

\end{example}

\noindent
We show that polar subsumption implies entailment. Note, the properties of negation we require to prove the following two theorems hold in the commonly used semantics for NAF (such as stable and well-founded) when unstratified usage of negation does not occur. Both stratified and polar logic programs do not allow for unstratified usage of negation.

\begin{theorem}[\textbf{Entailment property}]
\label{thm:EntPro}
Let $P$ and $Q$ be polar programs such that  $P\psub Q$. 
Then $P \models Q$.
\end{theorem}

\begin{proof} 
For any $r_2$  in $Q^{\even}$ the implication trivially follows from properties of $\leq_{\theta}$. Now let $r_1$ in $P^{\odd}$ and $r_2$ in $Q^{\odd}$  such that $r_2\leq_{\theta} r_1$; this implies $r_2\models r_1$. By contraposition, we derive   $\mathbf{not}\ r_1\models\mathbf{not}\ r_2$. 

\end{proof}

\noindent
We also show that polar programs can express all concepts expressible as stratified programs:



\begin{theorem}
\label{thm:stratified}
Let $S$ be a stratified logic program. Then there exists a polar program $P$ such that for all $p\in top(S)$, $S\models \forall \vec{x}.p(\vec{x})$ iff $P\models \forall \vec{x}.p(\vec{x})$ (See Appendix~\ref{app:stratpol}).
\end{theorem}

\noindent
Thus, we do not lose expressivity by learning polar programs rather than stratified programs.

\subsection{Learning From Failures (LFF) }
LFF searches a hypothesis space (a set of hypotheses) for a hypothesis that generalises examples and background knowledge.
In the existing literature, an LFF hypothesis is a definite (monotonic) program.
LFF uses \emph{hypothesis constraints} to restrict the hypothesis space.
Let $\mathcal{L}$ be a language that defines hypotheses.
A \emph{hypothesis constraint} is a constraint expressed in $\mathcal{L}$.
Let $C$ be a set of hypothesis constraints written in a language $\mathcal{L}$.
A hypothesis $H$ is \emph{consistent} with $C$ if, when written in $\mathcal{L}$, $H$ does not violate any constraint in $C$.
We denote as $\mathcal{H}_{C}$ the subset of the hypothesis space $\mathcal{H}$ which does not violate any constraint in $C$.

\subsection{LFFN}
We extend LFF to learn polar programs, which we call the \emph{learning from failures with negation (LFFN)} setting.
We define the LFFN input:
\begin{definition}[\textbf{LFFN input}]
\label{def:probin}
A \emph{LFFN} input is a tuple $(E^+, E^-, B, \mathcal{H}, C)$ where $E^+$ and $E^-$ are sets of ground atoms denoting positive and negative examples respectively; $B$ is a normal logic program denoting background knowledge;
$\mathcal{H}$ is a hypothesis space of polar programs, and $C$ is a set of hypothesis constraints.
\end{definition}

\noindent
To be clear, an LFFN hypothesis is a polar (non-monotonic) program and the hypothesis space is a set of polar programs.
We define an LFFN solution:

\begin{definition}[\textbf{LFFN solution}]
\label{def:solution}
Given an input tuple $(E^+, E^-, B, \mathcal{H}, C)$, a hypothesis $H \in \mathcal{H}_{C}$ is a \emph{solution} when $H$ is \emph{complete} ($\forall e \in E^+, \; B \cup H \models e$) and \emph{consistent} ($\forall e \in E^-, \; B \cup H \not\models e$).
\end{definition}

\noindent
If a hypothesis is not a solution then it is a \emph{failure}.
A hypothesis is \emph{incomplete} when $\exists e \in E^+, \; H \cup B \not \models e$; 
\emph{inconsistent} when $\exists e \in E^-, \; H \cup B \models e$;
 \emph{partially complete} when $\exists e \in E^+, \; H \cup B \models e$; 
 and \emph{totally incomplete} when $\forall e \in E^+, \; H \cup B \not \models e$.
 
Let $cost : \mathcal{H} \mapsto \mathbb{N}$ be a function that measures the cost of a hypothesis. 
We define an \emph{optimal} solution:

\begin{definition}[\textbf{Optimal solution}]
\label{def:opthyp}
Given an input tuple $(E^+, E^-, B, \mathcal{H}, C)$, a hypothesis $H \in \mathcal{H}_{C}$ is \emph{optimal} when (i) $H$ is a solution and (ii) $\forall H' \in \mathcal{H}_{C}$, where $H'$ is a solution, $cost(H) \leq cost(H')$.
\end{definition}

\noindent
 Our cost function is the number of literals in the hypothesis. 


\subsection{LFFN Constraints }
\label{subsec:LFFN}
An LFF learner learns hypothesis constraints from failed hypotheses.
\citet{popper} introduce hypothesis constraints based on subsumption.
A \emph{specialisation} constraint prunes specialisations of a hypothesis.
A \emph{generalisation} constraint prunes generalisations.
\label{subsec:NegGenSpe}
The existing LFF constraints are only sound for monotonic programs, i.e. they can incorrectly prune optimal solutions when learning non-monotonic programs, and are thus unsound for the LFFN setting.
The reason for unsoundness is that entailment is not a consequence of subsumption in a non-monotonic setting, even in the propositional case, as the following examples illustrate.



\begin{example}
\label{ex:genspe}
Consider the programs:

\begin{minipage}{.22\textwidth}
\[
 P =   \begin{array}{l}
    \left\{
    \begin{array}{l}
    \emph{a.}\\
    \emph{f $\leftarrow$ \textbf{not}  $inv_1$ }\\
    \emph{ $inv_1$  $\leftarrow$ b}\\
    \emph{ $inv_1$  $\leftarrow$ a}
    \end{array}
    \right\}
    \end{array}
\]
\end{minipage}
\begin{minipage}{.17\textwidth}
\[
 Q =   \begin{array}{l}
    \left\{
    \begin{array}{l}
    \emph{a.}\\
    \emph{f $\leftarrow$ \textbf{not}  $inv_1$ }\\
    \emph{ $inv_1$  $\leftarrow$ b}\\
    \end{array}
    \right\}
    \end{array}
\]
\end{minipage}
\vspace{.2em}

\noindent
Note that $P\sub Q$ and $Q \models f$ but $P\not \models f$.
Similarly, we have the following:
\vspace{-.4em}

\begin{minipage}{.22\textwidth}
\[
 P' =   \begin{array}{l}
    \left\{
    \begin{array}{l}
    \emph{a.}\\
    \emph{f $\leftarrow$ \textbf{not}  $inv_1$ }\\
    \emph{ $inv_1$  $\leftarrow$ a, b}
    \end{array}
    \right\}
    \end{array}
\]
\end{minipage}
\begin{minipage}{.17\textwidth}
\[
 Q' =   \begin{array}{l}
    \left\{
    \begin{array}{l}
    \emph{a.}\\
    \emph{f $\leftarrow$ \textbf{not}  $inv_1$ }\\
    \emph{ $inv_1$  $\leftarrow$ a}
    \end{array}
    \right\}
    \end{array}
\]
\end{minipage}
\end{example}
\noindent
Note that $Q'\sub P'$ and $P' \models f$ but $Q'\not \models f$.

To overcome this limitation,  we introduce constraints that are optimally sound for polar programs (Definition~\ref{def:polarity}) based on polar subsumption (Definition~\ref{def:polarsub}). 
In the LFFN setting, Theorem~\ref{thm:EntPro} implies the following propositions: 
\begin{proposition}[\textbf{Generalisation soundness}]
\label{prop:gen_sound}
Let $(E^+, E^-,$ $ B, \mathcal{H}, C)$ be a LFFN input, $H_1, H_2 \in \mathcal{H}_{C}$, $H_1$ be inconsistent, and $H_2\psub H_1$.
Then $H_2$ is not a  solution.
\end{proposition}
\begin{proposition}
[\textbf{Specialisation soundness}]
\label{prop:spec_sound}
Let $(E^+, E^-,$ $ B, \mathcal{H}, C)$ be a LFFN input, $H_1, H_2 \in \mathcal{H}_{C}$, $H_1$ be incomplete, and $H_1\psub H_2$.
Then $H_2$ is not a solution.
\end{proposition}

\noindent
To summarise, \textit{polar subsumption} allows us to soundly prune the hypothesis space when learning non-monotonic programs. 
In this next section, we introduce an algorithm that uses \textit{polar subsumption} to efficiently learn polar programs.

\section{Algorithm}
\label{sec:impl}

We now describe our \name{} algorithm.
To aid our explanation, we first describe \popper{} \cite{popper}.

\paragraph{\popper{}.}
\popper{} takes as input an LFF input\footnote{An LFF input is the same as an LFFN input except the hypothesis space contains only definite (monotonic) programs.} and learns hypotheses as definite programs without NAF.
To generate hypotheses, \popper{} uses an ASP program $P$ where each model (answer set) of $P$ represents a hypothesis. \popper{} follows a \textit{generate}, \textit{test}, and \textit{constrain} loop to find a solution. 
First, it generates a hypothesis as a solution to $P$ with the ASP system Clingo \cite{gebser2014}. Then, \popper{} tests this hypothesis given the background knowledge against the training examples, typically using Prolog. If the hypothesis is a solution, \popper{} returns it. Otherwise, the hypothesis is a failure: \popper{} identifies the kind of failure and builds constraints accordingly. For instance, if the hypothesis is inconsistent, \popper{} builds a generalisation constraint. 
\popper{} adds these constraints to $P$ to constrain subsequent \textit{generate} steps. 
This loop repeats until the solver finds a solution or there are no more models of $P$.

\subsection{\name{}} \name{} builds on \popper{} and follows a \textit{generate}, \textit{test}, and \textit{constrain} loop. 
The two key novelties of \name{} are its ability to (i) learn polar programs and (ii) use non-monotonic generalisation and specialisation constraints to efficiently prune the hypothesis space.
We describe these advances in turn.

\subsubsection{Polar Programs}
To learn polar programs, we extend the \textit{generate} ASP program to generate normal logic programs, i.e. programs with negative literals.
To only generate polar programs, we add the rules and constraints of Definitions~\ref{def:topset} and~\ref{def:polarity} to the ASP program to eliminate models where a predicate symbol has multiple polarities.
The complete ASP encoding is in  Appendix~\ref{sec:ASPencoding}, but we briefly explain it at a high-level.
If a predicate symbol $p$ occurs in the body of a rule with head symbol $q$ we say $q$ \textit{calls} $p$, which we name the \textit{call relation}. 
A predicate can be called positively or negatively. We compute the transitive closure of the call relation tracking the number of negative calls on each path. 
If a symbol has an even number of negative calls on a path to a top symbol we say its associated rules are positive, otherwise they are negative. 
If any rule is labeled both positive and negative then the program is non-polar. 
We ignore background knowledge predicates when computing the call relation.

\subsubsection{Polar Constraints} 
\name{} uses two types of constraints to prune models and thus prune hypotheses.
We refer to these constraints as \textit{polar specialisation} and \textit{polar generalisation} constraints. 
These constraints differ from those used by \popper{} because 
(i) they use additional literals to assign polarity to rules, and (ii) they use \emph{polar subsumption} (Definition \ref{def:polarsub}) rather than standard subsumption.
Polarity is important when learning polar programs because a \textit{polar generalisation} constraint prunes generalisations of positive polarity rules and specialisations of negative polarity rules.
A polar specialisation constraint prunes the specialisations of positive polarity rules and generalisations of negative polarity rules. 
\begin{example}
Reconsider the Zendo scenario from the introduction (Figures \ref{fig:motivatingExample} and \ref{fig:motivatingExample2}).
The following hypothesis is \textit{incomplete} as every positive example contains at least one cone: 
\[
    h_1 = \begin{array}{l}
    \left\{
    \begin{array}{l}
\emph{$r_1$ : f(S) $\leftarrow$ ${scene}(S)$, \textbf{not}\ ${inv_1}(S)$}\\
\emph{$r_2$ : \ $inv_1(S)$ $\leftarrow$ ${cone}(S,A)$}
    \end{array}
    \right\}
    \end{array}
\]
Since $h_1$ is incomplete, we can use a polar specialisation constraint to prune the hypothesis:
\[
    h_2 = \begin{array}{l}
    \left\{
    \begin{array}{l}
\emph{$r_1$ : 
f(S) $\leftarrow$ ${scene}(S)$, \textbf{not}\ ${inv_1}(S)$}\\
\emph{$r_2$ : \ $inv_1$(S) $\leftarrow$ ${cone}(S,A)$}\\
\emph{$r_3$ : \ $inv_1$(S) $\leftarrow$ ${contact}(S,A,\_)$, ${red}(A)$,}\\ \hspace{2em} \emph{\textbf{not}\ ${blue}(A)$}
    \end{array}
    \right\}
    \end{array}
\]
\noindent
The hypothesis $h_2$ is a superset of $h_1$ as it includes the additional rule $r_3$.
In $h_2$, the symbol $inv_1$ is negative because it is used negatively in $r_1$.
Therefore, the rule $r_3$ implies that $h_2$ is a specialisation of $h_1$.

The polar specialisation constraint also prunes $h_3$:
\[
    h_3 = \begin{array}{l}
    \left\{
    \begin{array}{l}
\emph{$r_4$ : 
f(S) $\leftarrow$ ${scene}(S)$, \textbf{not}\ ${inv_1}(S)$,} \\ \hspace{1.7cm} \emph{${contact}(S,A,\_)$,${red}(A)$, ${blue}(A)$}\\
\emph{$r_2$ : \ $inv_1$(S) $\leftarrow$ ${cone}(S,A)$}\\
    \end{array}
    \right\}
    \end{array}
\]
The rule $r_4$ in $h_3$ adds literals to $r_1$ in $h_1$, so $h_3$ is a specialisation of $h_1$.
By contrast, a polar specialisation constraint does not prune the following hypothesis:
\[
    h_4 = \begin{array}{l}
    \left\{
    \begin{array}{l}
\emph{$r_1$ : 
f(S) $\leftarrow$ ${scene}(S)$, \textbf{not}\ ${inv_1}(S)$}\\
\emph{$r_5$ : \ $inv_1$(S) $\leftarrow$ ${cone}(S,A)$, \textbf{not}\ ${red}(A)$}
    \end{array}
    \right\}
    \end{array}
\]
Notice that the rule $r_5$ has an additional negated body literal compared to $r_2$.
The symbol of this literal is neither positive nor negative, so we can ignore the occurrence of $\dneg$.
Thus, $h_4$ is a generalisation of $h_1$ as the new literal occurs in the body of $r_2$ whose head symbol has negative polarity.  
\end{example}

\section{Experiments}
\label{sec:exp}
To evaluate the impact of combining negation and PI, our experiments aim to answer the question:

\begin{description}
\item[Q1] Can negation and PI improve learning performance?
\end{description}

\noindent
To answer \textbf{Q1}, we compare the performance of \name{} against \popper{}, which cannot negate invented predicate symbols.
Comparing \name{} against different systems with different biases will not allow us to answer the question, as we would be unable to identify the reason for any performance difference.
To answer \textbf{Q1}, we use tasks where negation and PI should be helpful, such as learning the rules of Zendo \cite{DBLP:conf/cogsci/BramleyRTXG18}.
We describe the tasks in the next section.

We introduced sound constraints for polar programs to prune non-optimal solutions from the hypothesis space.
To evaluate whether these constraints improve performance, our experiments aim to answer the question:

\begin{description}
\item[Q2] 
Can polar constraints improve learning performance?
\end{description}

\noindent
To answer \textbf{Q2}, we compare the performance of \name{} with and without these constraints.

We introduced \name{} to go beyond existing approaches by combining negation and PI.
Our experiments, therefore, aim to answer the question:

\begin{description}
    \item[Q3] How does \name{} compare against existing approaches?
\end{description}

\noindent
To answer \textbf{Q3}, we compare \name{} against \popper{}, \ale{}, and \metagolsn{}. 
We describe these systems below. 

Questions \textbf{Q1-Q3} focus on tasks where negation and PI should help.
However, negation and PI are not always necessary.
In such cases, can negation and PI be harmful?
Our experiments try to answer the question:

\begin{description}
\item[Q4] Can negation and PI degrade learning performance?
\end{description}

\noindent
To answer \textbf{Q4}, we evaluate \name{} on tasks where negation and PI should be unnecessary.

\paragraph{Reproducibility.} 
Experimental code may be found in the following repository: \textit{github.com/Ermine516/NOPI}

\subsubsection*{Domains}
We briefly describe our domains. The precise problems are found in Appendix~\ref{sec:learnedproblems}.


\textbf{Basic (B).}
Non-monotonic learning tasks introduced by \citet{stepwise2018} and \citet{hopper}, such as learning the definition of a leap year.

\textbf{Zendo (Z).}
\citet{DBLP:conf/cogsci/BramleyRTXG18}
introduce \textit{Zendo} tasks similar to the one in the introduction.

\textbf{Graphs (G).} 
We use commonly used graph problems \cite{dilp,DBLP:conf/icml/GlanoisJFWZ0LH22}, such as \textit{dominating set}, \textit{independent set}, and \textit{connectedness}. 

\textbf{Sets (S).} 
These set-based tasks include \emph{symmetric difference}, \emph{decomposition into subsets}, and \emph{mutual distinctness}. 




\subsubsection*{Systems}
\label{sec:sysytems}
We compare \name{} against \popper{} \cite{popper,popper2}, \ale{} \cite{aleph}, and \metagolsn{} \cite{stepwise2018}.
We give \name{} and \popper{} identical input. 
The only experimental difference is the ability of \name{} to negate invented predicate symbols.
\noindent\ale{} can learn normal logic programs but uses a different bias than \name{} so the comparison should be interpreted as indicative only.
Also, we use the default \ale{} settings, but there are likely to be better settings on these datasets \cite{ashwinsettings}.
\metagolsn{} can learn normal logic programs but requires metarules to define the hypothesis space.
We use the metarules used by \citet{stepwise2018} supplemented with a general set of metarules \cite{reduce}.

\subsubsection*{Experimental Setup}
We use a 300s learning timeout for each task and round accuracies and learning times to integer values. 
We plot 99\% confidence intervals.
Additional experimental details are in Appendix~\ref{app:AED}. 

\textbf{Q1.} 
We allow all the systems to negate the given background relations.
For instance, in the Zendo tasks, each system can negate colours such \emph{red}.
Therefore, any improvements from \name{} are not from the use of negation but from the combination of negation and PI.

\textbf{Q2.} 
We need a baseline to evaluate our polar constraints.
As discussed in Section \ref{sec:setting}, \popper{} uses unsound constraints when learning polar programs.
If a program $h$ is not a solution and has a negated invented symbol, the only sound option for \popper{} is to prune $h$ from the hypothesis space, but, importantly, not its generalisations or specialisations.
To evaluate our polar constraints, we compare them against this simpler (banish) approach, which we call \nameb{}.
In other words, to answer \textbf{Q2}, we compare \name{} against \nameb{}.\footnote{
Table~\ref{tab:PredictiveAcc2}
in the Appendix compares sound and unsound constraints.}

\subsection{Results}

\subsubsection{\textbf{Q1. Can negation and PI improve performance?}}
Table \ref{tab:PredictiveAcc} shows the predictive accuracies of \name{} and \popper{}.
The results show that \name{} vastly outperforms \popper{} regarding predictive accuracies. 
For instance, for \emph{all red} (Z2) \popper{} learns:
  $$ \footnotesize \left\{  \begin{array}{l}
        \emph{z(A)$\leftarrow$piece(A,B), contact(B,C), red(C), rhs(C)}\\
        \emph{z(A)$\leftarrow$piece(A,B), contact(B,C), upright(C), lhs(B)}\\
        \emph{z(A)$\leftarrow$piece(A,B), contact(B,C), lhs(C), lhs(B)}\\
        \emph{z(A)$\leftarrow$piece(A,B), coord1(B,C), size(B,C), upright(B)}
    \end{array}  \right\}$$
\noindent
By contrast, \name{} learns:
$$  \left\{  \begin{array}{l}
       \emph{z(A)$\leftarrow$ scene(A), \textbf{not} $inv_1(A)$}\\
        \emph{$inv_1(A)\leftarrow$ piece(A,B), \textbf{not} red(B)}
    \end{array}  \right\}$$
\noindent
Table \ref{tab:times} shows the corresponding learning times.
The results show that \name{} rarely needs more than 40s to learn a solution. One of the more difficult problems (30s to learn) is \emph{largest is red} (Z6), which involves inventing two predicate symbols and having two layers of negation, which, as far as we are aware, goes beyond anything in the existing literature:
  $$  \left\{  \begin{array}{l}
       \emph{zendo(A) $\leftarrow$ {scene}(A), piece(A,B), \textbf{not}\ $inv_1(B,A)$}\\
        \emph{$inv_1(A,B)$ $\leftarrow$ piece(B,C), size(C,D), \textbf{not}\ $inv_2(D,A)$}\\
          \emph{$inv_2(A,B)$ $\leftarrow$ size(B,C), red(B), $A\leq C$}

    \end{array}  \right\}$$

\noindent
\popper{} sometimes terminates in less than a second.
The reason is that on some problems, because of its highly efficient search, \popper{} almost immediately proves that there is no monotonic solution.

Overall, the results from this section suggest that the answer to \textbf{Q1} is that combining negation and PI can drastically improve learning performance.

\begin{table}[ht]
\centering
\footnotesize
\begin{tabular}{@{}l|cccc@{}}
\textbf{Task} & 
  \textbf{\name{}} & \textbf{\popper} &\textbf{\ale{}} & \textbf{\metagolsn{}} \\
\midrule
B1 &\textbf{100  $\pm$ 0 } &  82  $\pm$ 0 & 50  $\pm$ 0  & 0  $\pm$ 0   \\ 
B2 & \textbf{100  $\pm$ 0 }  &  0  $\pm$ 0  & 50 $\pm$ 0  &  \textbf{100  $\pm$ 0 }  \\ 
B3 & \textbf{100  $\pm$ 0 }  &  82  $\pm$ 0 &  82 $\pm$ 0 & \textbf{100   $\pm$ 0 } \\ 
\hline

Z1 & \textbf{100  $\pm$ 0 } &  0  $\pm$ 0 & 60  $\pm$ 0  &  0 $\pm$  0   \\ 
Z2 & \textbf{100  $\pm$ 0 } &  55  $\pm$ 0  & 67 $\pm$ 0  & 0$\pm$  0   \\ 
Z3 & \textbf{100  $\pm$ 0 }  &  0  $\pm$ 0 & 65 $\pm$ 0  & 0 $\pm$  0   \\  
Z4 & \textbf{100  $\pm$ 0 } &  55  $\pm$ 0  & 58 $\pm$ 0  & 0$\pm$  0  \\ 
Z5 & \textbf{100  $\pm$ 0 } &   0  $\pm$ 0  & 21 $\pm$ 0  & 0 $\pm$  0  \\ 
Z6   & \textbf{100  $\pm$ 0 } &  0  $\pm$ 0 & 45 $\pm$ 0  & 0 $\pm$  0 \\  

\hline
G1 & \textbf{100  $\pm$ 0 } &  0  $\pm$ 0  & 50  $\pm$ 0  & 0  $\pm$ 0  \\ 
G2 & \textbf{100  $\pm$ 0 }  & 24  $\pm$ 0  & 47  $\pm$ 0  & 0  $\pm$ 0   \\ 
G3 & \textbf{100  $\pm$ 0 } &  0  $\pm$ 0  & 12 $\pm$ 0  & 0  $\pm$ 0   \\  
G4 & \textbf{100  $\pm$ 0 }  &  20  $\pm$ 0  & \textbf{100 $\pm$ 0}  & 0  $\pm$ 0  \\ 
G5 & \textbf{100  $\pm$ 0 }  &  0  $\pm$ 0 & 50  $\pm$ 0  & 0  $\pm$ 0  \\ 
G6 & \textbf{100  $\pm$ 0 } &  0  $\pm$ 0  & 21 $\pm$ 0  & 0  $\pm$ 0   \\ 
G7 &  \textbf{100  $\pm$ 0 } &  0  $\pm$ 0 & 50 $\pm$ 0  & 0  $\pm$ 0   \\ 
G8 & \textbf{100  $\pm$ 0 }  &  0  $\pm$ 0 & 50 $\pm$ 0  & 0  $\pm$ 0  \\ 
\hline

S1  & \textbf{100  $\pm$ 0 }  &  0  $\pm$ 0 & 50  $\pm$ 0  & 0  $\pm$ 0  \\ 
S2  & \textbf{100  $\pm$ 0 } &  0  $\pm$ 0  & 50  $\pm$ 0  & 0 $\pm$ 0   \\ 
S3  & \textbf{92  $\pm$ 0} &  0  $\pm$ 0  & 57  $\pm$ 0  & 0  $\pm$ 0  \\ 
S4  & \textbf{100  $\pm$ 0 } & 0 $\pm$ 0   & 50 $\pm$ 0  & 0  $\pm$ 0  \\ 
S5  &  \textbf{100  $\pm$ 0 }  & 57 $\pm$ 0 & 23 $\pm$ 0  & 0  $\pm$ 0 \\ 
S6  &  \textbf{100  $\pm$ 0 } & 0  $\pm$ 0  & 0  $\pm$ 0  & 0  $\pm$ 0 \\ 

\end{tabular}
\caption{
Mean predictive accuracies (10 runs).
}
\label{tab:PredictiveAcc}
\end{table}

\begin{table}[ht]
\centering
\footnotesize
\begin{tabular}{@{}l|cccc@{}}
\textbf{Task} & 
\textbf{\name{}} & 
\textbf{\popper{}} & 
\textbf{\ale{}} & 
\textbf{\metagolsn{}} \\
\midrule
B1 & 20  $\pm$ 0  & \emph{timeout}   & 20  $\pm$ 0  & \emph{timeout}   \\
B3          & 3  $\pm$ 0  & 0  $\pm$ 0   & 18 $\pm$ 2  & 1 $\pm$ 0  \\
\hline

Z1 & 2  $\pm$ 0 & 0  $\pm$ 0  & 7$\pm$ 1  & \emph{timeout}   \\ 
Z2 & 12  $\pm$ 0 & 1  $\pm$ 0  & 95 $\pm$ 2  & \emph{timeout}   \\
Z3 & 2  $\pm$ 0  & 0  $\pm$ 0  & 27 $\pm$ 1  & \emph{timeout}   \\ 
Z4 & 22  $\pm$ 1  & 0  $\pm$ 0  & 20 $\pm$ 1  & \emph{timeout}  \\
Z5 & 15 $\pm$ 1  & 0  $\pm$ 0  & 24 $\pm$ 1  & \emph{timeout}  \\ 
Z6 & 67 $\pm$ 4  & 0  $\pm$ 0  & 149 $\pm$ 24  & \emph{timeout} \\ 

\hline
G1 & 4  $\pm$ 0 & 0  $\pm$ 0   & 32  $\pm$ 2  & \emph{timeout}  \\
G2 & 2  $\pm$ 0 & 0  $\pm$ 0   & 0  $\pm$ 0  & \emph{timeout}   \\
G3 & 9 $\pm$ 0  & 16  $\pm$ 0    & 1  $\pm$ 0  & \emph{timeout}   \\ 
G4 & 12  $\pm$ 0  & 0  $\pm$ 0   & 1 $\pm$ 0  & \emph{timeout}  \\
G5 & 8  $\pm$ 0  & 0  $\pm$ 0   & 0  $\pm$ 0  & \emph{timeout}  \\
G6 & 19  $\pm$ 3  & 0  $\pm$ 0   & 12 $\pm$ 1  & \emph{timeout}   \\
G7 & 58  $\pm$ 8 & 0  $\pm$ 0   & 12 $\pm$ 1  & \emph{timeout}   \\
G8 & 71  $\pm$ 9 & 0 $\pm$ 0   & 38 $\pm$ 1  & \emph{timeout}  \\
\hline

S2 & 3 $\pm$ 0    & 0  $\pm$ 0    & 1  $\pm$ 0  & \emph{timeout}   \\
S4 & 28  $\pm$ 2   & 0  $\pm$ 0   & 0 $\pm$ 2  & 0    $\pm$ 0  \\
S5 & 43  $\pm$ 3  & 0  $\pm$ 0    & 23 $\pm$ 3  & \emph{timeout} \\
S6 & 3  $\pm$ 0   & 0  $\pm$ 0   & 1  $\pm$ 0  & \emph{timeout} \\

\end{tabular}
\caption{
Mean learning times (10 runs). 
We only show tasks where the times of \name{} and \popper{} differ by more than 1 second. 
}
\label{tab:times}
\end{table}
\subsubsection{\textbf{Q2. 
Can polar constraints improve performance?
}}

Table~\ref{tab:q2times} shows the learning times of \name{} and \nameb{}.
The results show that \name{} has lower learning times than \nameb{}.
In other words, the results show that polar constraints can drastically reduce learning times.
A \textit{wilcoxon signed-rank test} confirms the significance of the differences at the $p < 10^{-8}$ value.
For simpler tasks, there is little benefit from the polar constraints as the overhead of constructing and adding them to the solver negates the pruning benefits.
For more difficult tasks, the difference is substantial.
For instance, the learning times for \name{} and \nameb{} on the \textit{sym. difference} (S4) task are 31s and 72s respectively, a 57\% reduction.
Overall, the results suggest that the answer to \textbf{Q2} is that our polar constraints can drastically reduce learning times. 

\begin{table}[ht!]
\centering
\footnotesize
\begin{tabular}{@{}l|ccc@{}}
\textbf{Task} & 
\textbf{\name{}} & 
\textbf{\nameb{}} &  \textbf{Change}\\
\midrule
B3       & 2  $\pm$ 0  & 4 $\pm$ 0  & \textbf{-50\% }\\
\hline

Z4  & 11 $\pm$ 1  & 49 $\pm$  2  & \textbf{-78\%}\\
Z5  & 13 $\pm$ 1 & 29 $\pm$  1  & \textbf{-55\%}\\ 
Z6  & 30 $\pm$ 1  & 115 $\pm$  9  & \textbf{-74\%}\\ 
\hline

G2                & 1  $\pm$ 0  & 9 $\pm$ 0 & \textbf{-89\%} \\
G3             & 18  $\pm$ 1  & 23 $\pm$ 1 & \textbf{-22\%} \\ 
G4        & 20 $\pm$ 3  & 68 $\pm$ 4  & \textbf{-71\%} \\
G5                 & 3 $\pm$ 0  & 11 $\pm$ 0  & \textbf{-73\%} \\
G6         & 23 $\pm$ 1  & 33 $\pm$ 3  & \textbf{-27\%} \\
G7            & 56 $\pm$ 5  & 93 $\pm$ 9    & \textbf{-40\%} \\
G8       & 63 $\pm$ 5  & 103 $\pm$ 9   & \textbf{-39\%}\\
\hline
S3         & 3  $\pm$ 0  & 13  $\pm$ 0  & \textbf{-77\%} \\
S4        & 31 $\pm$ 2  & 72 $\pm$ 9 & \textbf{-57\%}\\
S5          & 35 $\pm$ 2  & 53 $\pm$ 5  & \textbf{-34\%}\\
S6       & 4  $\pm$ 0  & 8  $\pm$ 1  & \textbf{-50\%} \\
\end{tabular}
\caption{
Mean learning time (125 runs). 
We only show tasks where the times differ by more than 1 second.
}
\label{tab:q2times}
\end{table}

\subsubsection{\textbf{Q3. How does \name{} compare against existing approaches?}}

Table \ref{tab:PredictiveAcc} shows the predictive accuracies of the systems.
As is clear, \name{} overwhelmingly outperforms the other systems.
This result is expected.
Besides \metagolsn{}, the other systems cannot learn normal logic programs with PI.
\ale{} can learn programs with NAF and sometimes learns reasonable solutions.
However, \ale{} cannot perform PI so, due to its restricted language, it struggles to generalise.  
In many cases, \ale{} simply memorises the training examples.
Because it relies on user-supplied metarules, \metagolsn{} can only learn normal logic programs 
of a very restricted syntactic structure and thus struggles on almost all our tasks.
Overall, the results from this section suggest that the answer to \textbf{Q3} is that \name{} performs well compared to other approaches on problems that need negation and PI.

\subsubsection{\textbf{Q4. Can negation and PI degrade performance?}}
\label{sec:q4}
The Appendix includes tables showing the predictive accuracies and learning times of the systems.
The results show that \name{} performs worse than \popper{} on these tasks.
The Blumer bound \cite{blumer:bound} helps explain why.
According to the bound, given two hypotheses spaces of different sizes, searching the smaller space should result in higher predictive accuracy compared to searching the larger one if the target hypothesis is in both.
\name{}  considers programs with negation and PI and thus searches a drastically larger hypothesis space than \popper{} and the other systems.
The tasks in \textbf{Q4} do not need negation and PI, thus explaining the difference. 

\section{Conclusions and Limitations}
We have introduced an approach that combines negation and PI.
Our approach can learn \textit{polar} programs, including stratified Datalog programs (Theorem \ref{thm:stratified}).
We introduced generalisation and specialisation constraints for this non-monotonic fragment and showed that they are optimally sound (Theorem~\ref{thm:EntPro}).
We introduced \name{}, an ILP system that can learn normal logic programs with PI, including recursive programs.
We have empirically shown on multiple domains that (i) \name{} can outperform existing approaches, and (ii) our non-monotonic constraints can reduce learning times.

\subsection*{Limitations and Future Work}


\textbf{Inefficient constraints.}
\name{} sometimes spends 30\%  of learning time building polar constraints.
This inefficiency is an implementation issue rather than a theoretical one.
Therefore, our empirical results likely underestimate the performance of \name{}, especially the improvements from using polar constraints. 

\textbf{Unnecessary negation and PI.}
Our results show that combining negation and PI allows \name{} to learn programs that other approaches cannot.
However, the results also show that this increased expressivity can be detrimental when the combination of negation and PI is unnecessary.
Thus, the main limitation of this work and direction for future work is to automatically detect when a problem needs negation and PI.
\clearpage

\section*{Acknowledgements}
The first author is supported by the \textit{Math}$_{LP}$ project (LIT-2019-7-YOU-213) of the Linz  Institute of Technology and the state of Upper Austria, Cost Action CA20111 \textit{EuroProofNet}, and Czech Science Foundation Grant No. 22-06414L, \textit{PANDAFOREST}.  The second author is supported by the EPSRC fellowship \textit{The Automatic Computer Scientist} (EP/V040340/1).
We thank Filipe Gouveia, Céline Hocquette, and Oghenejokpeme Orhobor for feedback on the paper.

\bibliography{references}

\clearpage

\appendix

 \newpage
\section{Terminology}
\label{sec:terminology}
\subsection{Logic Programming}
We assume familiarity with logic programming \cite{lloyd:book} but restate some key relevant notation. A \emph{variable} is a string of characters starting with an uppercase letter. A \emph{predicate} symbol is a string of characters starting with a lowercase letter. The \emph{arity} $n$ of a function or predicate symbol is the number of arguments it takes. An \emph{atom} is a tuple $p(t_1, ..., t_n)$, where $p$ is a predicate of arity $n$ and $t_1$, ..., $t_n$ are terms, either variables or constants. An atom is \emph{ground} if it contains no variables. A \emph{literal} is an atom or the negation of an atom. A \emph{clause} is a set of literals.
A \emph{clausal theory} is a set of clauses. A \emph{constraint} is a clause without a non-negated literal. A \emph{definite} clause is a clause with exactly one non-negated literal. A \emph{program} is a set of definite clauses. A \emph{substitution} $\theta = \{v_1 / t_1, ..., v_n/t_n \}$ is the simultaneous replacement of each variable $v_i$ by its corresponding term $t_i$. 
A clause $c_1$ \emph{subsumes} a clause $c_2$ if and only if there exists a substitution $\theta$ such that $c_1 \theta \subseteq c_2$. 
A program $h_1$ subsumes a program $h_2$, denoted $h_1 \preceq h_2$, if and only if $\forall c_2 \in h_2, \exists c_1 \in h_1$ such that $c_1$ subsumes $c_2$. A program $h_1$ is a \emph{specialisation} of a program $h_2$ if and only if $h_2 \preceq h_1$. A program $h_1$ is a \emph{generalisation} of a program $h_2$ if and only if $h_1 \preceq h_2$.
\subsection{Answer Set Programming}
We also assume familiarity with answer set programming \cite{asp} but restate some key relevant notation \cite{ilasp}.
A \emph{literal} can be either an atom $p$ or its \emph{default negation} $\text{not } p$ (often called \emph{negation by failure}). A normal rule is of the form $h \leftarrow b_1, ..., b_n, \text{not } c_1,... \text{not } c_m$. where $h$ is the \emph{head} of the rule, $b_1, ..., b_n, \text{not } c_1,... \text{not } c_m$ (collectively) is the \emph{body} of the rule, and all $h$, $b_i$, and $c_j$ are atoms. A \emph{constraint} is of the form $\leftarrow b_1, ..., b_n, \text{not } c_1,... \text{not } c_m.$ where the empty head means false. A \emph{choice rule} is an expression of the form $l\{h_1,...,h_m\}u \leftarrow b_1,...,b_n, \text{not } c_1,... \text{not } c_m$ where the head $l\{h_1,...,h_m\}u$ is called an \emph{aggregate}. In an aggregate, $l$ and $u$ are integers and $h_i$, for $1 \leq i \leq m$, are atoms. An \emph{answer set program} $P$ is a finite set of normal rules, constraints, and choice rules. Given an answer set program $P$, the \emph{Herbrand base} of $P$, denoted
as ${HB}_P$, is the set of all ground (variable free) atoms that can be formed from the predicates and constants that appear in $P$. When $P$ includes only normal rules, a set $A \in {HB}_P$ is an \emph{answer set} of $P$ iff it is the minimal model of the  \emph{reduct} $P^A$, which is the program constructed from the grounding of $P$ by first removing any rule whose body contains a literal $\text{not } c_i$ where $c_i \in A$, and then removing any defaultly negated literals in the remaining rules. An answer set $A$ satisfies a ground constraint $\leftarrow b_1, ..., b_n, \text{not } c_1,... \text{not } c_m.$ if it is not the case that $\{b_1, ..., b_n\} \in A$ and $A \cap \{c_1, ..., c_m\} = \emptyset$.

\section{Additional Experimental Details}
\label{app:AED}
 We enforce a timeout of 10 minutes per task. We measure predictive accuracy and learning time. We measure the mean and standard error over 10 trials. We use an 8-Core 1.6 GHz Intel Core i5 and a single CPU.
\section{Theorem~\ref{thm:stratified} Proof: Stratified to Polar}
\label{app:stratpol}
In this section, we show that stratified normal logic programs can be transformed into polar programs. Before defining what a stratified program is, we need a few definitions. Given a normal logic program $P$, we define $H(P)= top(P)\cup defs(P)$, and for any rule $r\in P$, the head symbol of $r$ will be denoted by $head_s(r)$. Stratified normal logic programs are defined as follows: 

\begin{definition}
\label{def:strat}
A  normal logic program program $S$ is stratified if there exists a total function $\mu:H(S)\rightarrow \mathbb{N}$
such that for all $p\in H(S)$ and $r\in S$: 
\begin{itemize}
    \item if $p\in \bodyv{r}{\even}$, then $\mu(head_s(r))\geq \mu(p)$
    \item if $p\in \bodyv{r}{\odd}$, then $\mu(head_s(r)) > \mu(p)$
\end{itemize}
\end{definition}

This definition prunes normal logic programs with unstratified use of negation, for example, $odd(X)\leftarrow succ(Y,X),\mathbf{not} \ odd(Y)$. Notice that Definition~\ref{def:strat} requires $\mu(odd) > \mu(odd)$, but this is impossible. To transform a stratified normal logic program $S$ into a polar program $S'$, we enforce the following properties on $S'$: 

\begin{itemize}
    \item For every $p \in  H(S')$,  $p$ is annotated exclusively $pos(p)$ or $neg(p)$ within $S'$
\end{itemize}
Notice that if we annotate a stratified program $S$, every symbol in $H(S)$ will be annotated. Only programs with unstratified negation contain symbols that cannot be annotated. 

To simplify the arguments in the proof below, we will define $ann(p)\in\{\even,\odd\}$ for $p \in  H(S')$, which denotes that where $pos(p)$ holds ($\even$) and/or $neg(p)$ holds($\odd$). Additionally, we define the function  $\parOp: \{\even,\odd\}\rightarrow \{\even,\odd\}$ which behaves as follows: 
\begin{equation*}
    x\parOp y= \left\lbrace\begin{array}{cc}
            \even & x=y \\
        \odd & x\not =y 
    \end{array}\right.
\end{equation*}

In addition, we require the following concepts: 
\begin{definition}
Let $S$ be a set of rules. Then $$\symT{S} = \{\symc{c} \vert c\in S\}.$$
\end{definition}

\begin{definition}
Let $S$ be a set of rules and $p$ and $q$ predicate symbols. Then $S[p\setminus q]$ denotes the set of rules where every occurrence of $p$ is replaced by an occurrence of $q$.
\end{definition}

We transform stratified programs into polar programs by introducing fresh names and duplicating rules. We illustrate the process below and show that polar programs are indeed as expressive as stratified normal logic programs.

\subsection{Flattening }
\begin{definition}
Let $S$ be a stratified  program, $c_1,c_2\in S$, and $p\in H(S)$  such that $p\in \bodyv{c_1}{\even}$,  $p\in \bodyv{c_2}{\odd}$, and $\symc{c_1} \not = p$. Then the program $S' =$
$$S\setminus \{c_1\} \cup \{c_1\}[p\setminus p']\cup \{c\vert \symc{c}=p \wedge c\in S\}[p\setminus p'] $$
is a $p$-\textit{flattening} of $S$ where $p'$ is fresh in $S$. We denote the $p$-flattening of $S$ into $S'$ by $S\rightharpoonup_p S'$ and let $\mathit{flat}(S)$ denote the set of $p\in H(S)$ for which a $p$-flattening of $S$ exists. 
\end{definition}
\begin{example}
\label{ex:flattening}
Note that for the following program $S$, $\mathit{flat}(S) = \{\mathtt{inv_1}\}$ and
\begin{align*}
    \mathtt{f}\leftarrow\ & \mathtt{inv_1},\mathbf{not}\ \mathtt{inv_1}.\\
    \mathtt{inv_1}\leftarrow\ & \mathtt{inv_3}.\\
    \mathtt{inv_3}\leftarrow\ & \mathtt{p}.
\end{align*}
the  $\mathtt{inv_1}$-flattening of $S$ results in the program $S'$ as follows
\begin{align*}
    \mathtt{f}\leftarrow\ & \mathtt{inv_2},\mathbf{not}\ \mathtt{inv_1}.\\
    \mathtt{inv_1}\leftarrow\ & \mathtt{inv_3}.\\
    \mathtt{inv_2}\leftarrow\ & \mathtt{inv_3}.\\
    \mathtt{inv_3}\leftarrow\ & \mathtt{p}.
\end{align*}
\end{example}

\begin{lemma}
Let $S,S'$ be stratified  programs such that $S\rightharpoonup_p S'$  and $p'\in \symT{S'}\setminus\symT{S}$. Then $p'\not \in\mathit{flat}(S')$. 
\end{lemma}
\begin{proof}
The symbol $p'$ only occurs positively in $S'$ and thus cannot be a member of $\mathit{flat}(S')$ by definition. 
\end{proof}

\begin{lemma}
\label{lem:step}
Let $S_0$ be a stratified program such that $p\in  \mathit{flat}(S)$. Then there $\exists n>0$ such that  $S_0\rightharpoonup_p S_1\rightharpoonup_p\cdots \rightharpoonup_p S_n$ and $p\not \in  \mathit{flat}(S_n)$.
\end{lemma} 
\begin{proof}
The process of $p$-flattening replaces one occurrence of  $p\in \bodyv{c}{\even}$ for $c\in S_i$, such that $\symc{c}\not= p$, by a fresh symbol thus $S_{i+1}$ has one less occurrence of $p$. After finitely many steps $n$ every occurrence of $p\in \body{c}$ for $c\in S_n$  such that $\symc{c}\not= p$ will occur in $p\in \bodyv{c}{\odd}$.
\end{proof}
We say $S>_f S'$ if there exists $p\in \mathit{flat}(S)$ such that  $S\rightharpoonup_p S'$. The transitive-reflexive closure is denoted by $\geq_f^*$.

\begin{lemma}
\label{lem:flatpol}
Let $S$ be a stratified program. Then there exists  $S'$ such that  $S \geq_f^* S'$ and  $\mathit{flat}(S') = \emptyset$. 
\end{lemma}
\begin{proof}
Follows from induction on the size of  $\mathit{flat}(S)$. The basecase is trivial, and the stepcase follows from Lemma~\ref{lem:step}.
 \end{proof}

We refer to a stratified program $S$ such that $\mathit{flat}(S) = \emptyset$ as \textit{semi-polar}. 
\begin{example}
\label{ex:app:semipole1}
The program from Example~\ref{ex:flattening} is \textit{semi-polar} but  not polar, that is
\begin{align*}
    \mathtt{f}\leftarrow\ & \mathtt{inv_2},\mathbf{not}\ \mathtt{inv_1}.\\
    \mathtt{inv_1}\leftarrow\ & \mathtt{inv_3}.\\
    \mathtt{inv_2}\leftarrow\ & \mathtt{inv_3}.\\
    \mathtt{inv_3}\leftarrow\ & \mathtt{p}.
\end{align*}
\end{example}
Semi-polar programs have property (1) mentioned at the beginning of this section. In the following subsection, we show how to transform semi-polar programs into programs that also have property (2). 

\subsection{Stretching}
In order to transform semi-polar programs into polar ones, we need to rename symbols with multiple \textit{trace} values:

\begin{definition}
Let $S$ be semi-polar, $x\in\{\even,\odd\}$ $r\in S$, $p\in H(S)$, and $l$ a literal of $r$ such that the symbol of $l$ is $p$. We define the trace of $p$ from $r$, denoted $\mathit{tr}_S(r,p)$, as follows:
\begin{itemize}
    \item  if  $p\in \bodyv{r}{x}$ and for all  $r'\in S$, $\symc{r}\not \in \body{r'}$,\\
    then $\mathit{tr}_S(r,p) = x$
    \item  if  $p\in \bodyv{r}{x}$, $\symc{r}\in \body{r'}$ and $p\not = \symc{r}$,  \\
    then $\mathit{tr}_S(r,p) = x\parOp \mathit{tr}_S(r',\symc{r})$
    \item if  $p= \symc{r}$, $p\in \body{r'}$, and  $\symc{r'}\not = p$, \\
    then $\mathit{tr}_S(r,p) = \mathit{tr}_S(r',p)$
    \item otherwise, $\mathit{tr}_S(r,p) = \even.$
\end{itemize}
\end{definition}
\begin{example}
\label{ex:traces}
The program from Example~\ref{ex:app:semipole1} has the following traces
$$ \mathit{tr}_S(4,\mathtt{inv_3}) =\mathit{tr}_S(3,\mathtt{inv_3}) = \even \parOp\mathit{tr}_S(1,\mathtt{inv_2}) = \even $$
$$ \mathit{tr}_S(4,\mathtt{inv_3}) =\mathit{tr}_S(2,\mathtt{inv_3}) = \even \parOp \mathit{tr}_S(1,\mathtt{inv_1}) = \odd $$
$$ \mathit{tr}_S(3,\mathtt{inv_3}) = \even \parOp \mathit{tr}_S(1,\mathtt{inv_2}) = \even \parOp  \even =\even $$
$$ \mathit{tr}_S(2,\mathtt{inv_3}) = \even \parOp \mathit{tr}_S(1,\mathtt{inv_1}) = \even \parOp  \odd =\odd $$
$$ \mathit{tr}_S(3,\mathtt{inv_2}) = \mathit{tr}_S(1,\mathtt{inv_2}) =\even $$
$$ \mathit{tr}_S(2,\mathtt{inv_1}) = \mathit{tr}_S(1,\mathtt{inv_1}) =\odd $$
where rules are numbered as below: 
\begin{align*}
   1:\mathtt{f}\leftarrow\ & \mathtt{inv_2},\mathbf{not}\ \mathtt{inv_1}.\\
   2:\mathtt{inv_1}\leftarrow\ & \mathtt{inv_3}.\\
   3:\mathtt{inv_2}\leftarrow\ & \mathtt{inv_3}.\\
   4:\mathtt{inv_3}\leftarrow\ & \mathtt{p}.
\end{align*}
Notice that $\mathtt{inv_3}$ has multiple trace values depending on where one starts the trace. 
\end{example}

Note that for both stratified and semi-polar programs, $\mathit{tr}_S(c,p)$ may return multiple values for the same symbol $p$ as the symbol of $c$ influences the result. For semi-polar programs, $\mathit{tr}_S(c,p)$ where $\symc{c} = p$ may also be problematic as the value of  $\mathit{tr}_S(c,p)$ is dependent on which rule $c'$ we choose. Thus, we avoid recursive rules when defining  $p$-\textit{stretching} below.

\begin{definition}
Let $S$ be semi-polar, $c_1,c_2\in S$, and $p\in H(S)$  such that $p\in \body{c_1}\cap\body{c_2}$,  $\mathit{tr}_S(c_1,p) \not = \mathit{tr}_S(c_2,p)$ , and $\symc{c_1} \not = \symc{c_2} \not = p$. Then $S' =$
$$S\setminus \{c_1\} \cup \{c_1\}[p\setminus p']\cup \{c\vert \symc{c}=p \wedge c\in S\}[p\setminus p'] $$
is a $p$-\textit{stretching} of $S$ where $p'$ is fresh in $S$. We denote the $p$-stretching of $S$ into $S'$ by $S\rightharpoondown_p S'$ and let $\mathit{stret}(S)$ denote the set of $p\in H(S)$ for which a $p$-stretching of $S$ exists. 
\end{definition}

\begin{example}
\label{ex:app:semipole2}
The program from Example~\ref{ex:traces}, denoted as $S$,  has $\mathit{stret}(S) = \{\mathtt{inv_3}\}$. The $\mathtt{inv_3}$-stretching results in the following program
\begin{align*}
   \mathtt{f}\leftarrow\ & \mathtt{inv_2},\mathbf{not}\ \mathtt{inv_1}.\\
   \mathtt{inv_1}\leftarrow\ & \mathtt{inv_3}.\\
   \mathtt{inv_2}\leftarrow\ & \mathtt{inv_4}.\\
   \mathtt{inv_3}\leftarrow\ & \mathtt{p}.\\
   \mathtt{inv_4}\leftarrow\ & \mathtt{p}.
\end{align*}
Notice that the resulting program has unique trace values for each symbol, and the program is polarized. For larger programs, this process requires more steps. We formalize it using the following lemmas. 
\end{example}

\begin{lemma}
Let $S,S'$ be semi-polar such that $S\rightharpoondown_p S'$  and $p'\in \symT{S'}\setminus\symT{S}$. Then $p'\not \in\mathit{stret}(S')$. 
\end{lemma}
\begin{proof}
$p'$ only occurs once and thus has a unique trace. 
\end{proof}
\begin{lemma}
\label{lem:step2}
Let $S$  be semi-polar such that $p\in  \mathit{stret}(S)$. Then there $\exists n>0$ such that  $S\rightharpoondown_p S_1\rightharpoondown_p\cdots \rightharpoondown_p S_n$ and $p\not \in  \mathit{stret}(S_n)$.
\end{lemma} 
\begin{proof}
The process of  $p$-stretching  replaces one occurrence of $p\in \bodyv{c}{\even}$ for $c\in S_i$, such that $\symc{c}\not= p$, by a fresh symbol thus $S_{i+1}$ has one less occurrence of $p$. After finitely many steps $n$, every occurrence of $p\in \body{c}$ for $c\in S_n$  will have the same trace.
\end{proof}
We say $S>_s S'$ if there exists $p\in \mathit{stret}(S)$ such that  $S\rightharpoondown_p S'$. The transitive-reflexive closure is denoted by $\geq_s^*$.

\begin{lemma}
\label{lem:semipol}
Let $S$ be semi-polar. Then there exists  $S'$ such that $S \geq_s^* S'$ and  $\mathit{stret}(S') = \emptyset$. 
\end{lemma}
\begin{proof}
Follows from induction on the size of  $\mathit{stret}(S)$. The basecase is trivial, and the stepcase follows from Lemma~\ref{lem:step2}.
 \end{proof}

 \begin{lemma}
Let $S,S'$ be semi-polar programs such that $S \geq_s^* S'$ and  $\mathit{stret}(S') = \emptyset$. Then $S'$ is polar. 
\end{lemma}
\begin{proof} It is easy to verify that $ann(p) =  \mathit{tr}_{S'}(c,p)$ where $p\in\body{c}$ and $p\in H(S')$.
 \end{proof}
Let $S$ be a stratified normal program, $S'$ a semi-polar program such that $S \geq_f^* S'$  and $S''$ a polar program such that  $S' \geq_s^* S''$. Then we refer to $S''$ as the polarisation of $S$, denoted $\mathit{pol}(S)$. 

 \begin{theorem}
Let $S$ be a stratified normal program and $p\in H(S)$. Then $S\models \forall \vec{x}.p(\vec{x})$ iff $\mathit{pol}(S)\models \forall \vec{x}.p(\vec{x})$.
\end{theorem}
\begin{proof}
This follows from the fact that flattening and stretching only introduces fresh symbols that duplicate existing predicate definitions. Thus, $\mathit{pol}(S)$ contains many repetitions, modulo renaming, of the predicate definitions contained in  $S$.
\end{proof}
The formulation in the main body of the paper follows from the above theorem as $top(S)\subseteq H(S)$ and $P=\mathit{pol}(S)$.

\begin{theorem}
Let $S$ be a stratified logic program. Then there exists a polar program $P$ such that for all $p\in top(S)$, $S\models \forall \vec{x}.p(\vec{x})$ iff $P\models \forall \vec{x}.p(\vec{x})$.
\end{theorem}

\section{Problems Table~\ref{tab:q2times}: Solutions}
\label{sec:learnedproblems}
Here we provide found solutions for the problems found in Table~\ref{tab:q2times}. We 
include the solutions found by \name{}, \nameb{}, \textbf{Aleph}, and \textbf{Metagol}$_{sn}$. 
Note \textbf{Metagol}$_{sn}$ only found solutions for Stepwise-narrowing problems.
\subsection{Basic}
 \begin{itemize}



\item \textbf{B1}:\underline{Divides Entire List} is there a number in the list A which divides every number in the list
    \begin{center}
        \textbf{\name{} \& \nameb{}}
    \end{center}
    \begin{verbatim}
divlist(A):- member(B,A), not inv1(B,A).
 inv1(A,B):- member(C,B), not my_div(A,B).
    \end{verbatim}
       \begin{center}
        \textbf{\popper}
    \end{center}
    \begin{verbatim}
divlist(A):- head(A,C),tail(A,B),
member(D,B), my_div(C,D).
    \end{verbatim}

    \begin{center}
        \textbf{Aleph}
    \end{center}
\begin{verbatim}
 divlist([11,33,44,121]).
 divlist([6,9,18,3,27]).
 divlist([2,4,6,8,10,12,14]).
    \end{verbatim}
\end{itemize}
\subsection{Step-Wise Narrowing Tasks}
\begin{itemize}
    \item \textbf{B2}: \underline{1 of 2 even} is one of the numbers A or B even
   
    \begin{center}
        \textbf{\name{} \& \nameb{}}
    \end{center}
    \begin{verbatim}
 one_even(A,B) :- even(A), not inv1(B,A).
 one_even(A,B) :- even(B), not inv1(B,A).
     inv1(A,B) :- even(B),even(A).
    \end{verbatim}

    \begin{center}
        \textbf{Aleph}
    \end{center}
\begin{verbatim}
one_even(3,18). one_even(2,5).
one_even(3,4).  one_even(2,3).
one_even(3,2).  one_even(1,4).
one_even(1,2).
    \end{verbatim}
\begin{center}
        \textbf{Metagol}$_{sn}$
    \end{center}
\begin{verbatim}
  one_even(A,B) :-  even(A), not _one_even(A,B).
  one_even(A,B) :-  even(B), not __one_even(A,B).
 _one_even(A,B) :- even(B).
__one_even(A,B) :- even(A).
\end{verbatim}
    \item \textbf{B3}: \underline{Leapyear} is A a leap year
    \begin{center}
        \textbf{\name{} \& \nameb{}}
    \end{center}
    \begin{verbatim}
leapyear(A):- divisible4(A), not inv1(A).
inv1(A):-  divisible100(A), not inv2(A).
inv2(A):- divisible400(A).
    \end{verbatim}

    \begin{center}
        \textbf{Aleph}
    \end{center}
\begin{verbatim}
leapyear(996). leapyear(988).  leapyear(984).  
leapyear(980).  leapyear(972).  leapyear(968).  
leapyear(964).  leapyear(956).  leapyear(952).  
leapyear(948).  leapyear(940).  leapyear(936).  
leapyear(932).  leapyear(924).  leapyear(920).  
leapyear(916).  leapyear(908).  leapyear(904). 
... '341 more positive instances'
leapyear(A):-divisible16(A). 
leapyear(1004).
\end{verbatim}
\begin{center}
        \textbf{Metagol}$_{sn}$
    \end{center}
\begin{verbatim}
  leapyear(A) :- div4(A), not _leapyear(A).
 _leapyear(A) :- div100(A), not __leapyear(A)).
__leapyear(A) :- div400(A).
\end{verbatim}
\end{itemize}
\subsection{Zendo}
\begin{itemize}
  \item \textbf{Z1}: \underline{Nothing is upright}  none of the cones in the scene A are upright.
   \begin{center}
        \textbf{\name{} \& \nameb{}}
    \end{center}
    \begin{verbatim}
 zendo(A) :- not inv1(A).
  inv1(A) :- piece(A,B),upright(B).
    \end{verbatim}

    \begin{center}
        \textbf{Aleph}
    \end{center}
\begin{verbatim}
zendo(A):-piece(A,B),size(B,C),small(C),
blue(B),not upright(B).

zendo(D):-piece(D,E),lhs(E),blue(E),
piece(D,F),green(F).
    \end{verbatim}
    
    \item \textbf{Z2}: \underline{All red cones} are all cones in the scene A red
   
   \begin{center}
        \textbf{\name{} \& \name$_{bn}$}
    \end{center}
    \begin{verbatim}
zendo(A) :- not inv1(A).
 inv1(A) :- piece(A,B), not red(B).
    \end{verbatim}
   \begin{center}
        \textbf{\popper}
    \end{center}
    \begin{verbatim}
zendo(A):- piece(A,B),contact(B,C),
red(C), rhs(C).

zendo(A):- piece(A,B),contact(B,C),
upright(C), lhs(B).

zendo(A):- piece(A,B),contact(B,C),
lhs(C), lhs(B).

zendo(A):- piece(A,B),coord1(B,C),
size(B,C), upright(B).
    \end{verbatim}

    \begin{center}
        \textbf{Aleph}
    \end{center}
\begin{verbatim}
 zendo(10).  zendo(4).
 zendo(15).  zendo(14).  zendo(12).  
 zendo(A):-piece(A,B),coord1(B,C),
 coord2(B,C),large(C),red(B).  
 
 zendo(D):-piece(D,E),coord1(E,F),
 coord2(E,F), red(E),lhs(E).    
 
 zendo(G):-piece(G,H),coord1(H,I),large(I),
 size(H,J),small(J),rhs(H).  

 zendo(K):-piece(K,L),coord2(L,M),size(L,M),
 upright(L),not blue(L).  
 
 zendo(N):-piece(N,O),contact(O,P),red(P),
 rhs(P).  
 
 zendo(Q):-piece(Q,R),coord2(R,S),size(R,S),
 red(R),strange(R).
 
    \end{verbatim}
  
    \item \textbf{Z3}: \underline{All same size} are all the cones in the scene A are the same size
    
   \begin{center}
        \textbf{\name{} \& \nameb{}}
    \end{center}
    \begin{verbatim}
 zendo(A) :- piece(A,C),size(C,B),
 not inv1(A,B).

inv1(A,B) :- piece(A,C),not size(C,B).
    \end{verbatim}

    \begin{center}
        \textbf{Aleph}
    \end{center}
\begin{verbatim}
zendo(A):-piece(A,B),red(B),lhs(B),
piece(A,C),green(C),strange(C).

zendo(D):-piece(D,E),upright(E),blue(E),
piece(D,F),green(F),rhs(F).

zendo(G):-piece(G,H),size(H,I),strange(H),
piece(G,J),size(J,I),rhs(J).

zendo(K):-piece(K,L),coord2(L,M),medium(M),
piece(K,N),red(N),upright(N).

\end{verbatim}
    \item \textbf{Z4}: \underline{Exactly a blue} is there exactly 1 blue cone in the scene A
   \begin{center}
        \textbf{\name{} \& \nameb{}}
    \end{center}
    \begin{verbatim}
   zendo(A) :- piece(A,B),not inv1(A,B),blue(B).
  inv1(A,B) :- piece(A,C),blue(C),not eq(B,C).
    \end{verbatim}
       \begin{center}

     \textbf{\popper}
    \end{center}
    \begin{verbatim}
  zendo(A):- piece(A,B),contact(B,C),
  strange(C).

zendo(A):- piece(A,B),contact(B,C),
not green(C).
    \end{verbatim}

    \begin{center}
        \textbf{Aleph}
    \end{center}
\begin{verbatim}
zendo(A) :- piece(A,B),coord1(B,C),large(C),
coord2(B,D),small(D),upright(B).

zendo(E) :- piece(E,F),coord1(F,G),rhs(F),
piece(E,H),coord2(H,G),red(H).
\end{verbatim}
    \item \textbf{Z5}: \underline{All blue or small} are all the cones in the scene A blue or small
   \begin{center}
        \textbf{\name{} \& \nameb{}}
    \end{center}
    \begin{verbatim}
zendo(A):- not inv1(A).
inv1(A) :- piece(A,B),not inv2(B).
inv2(A) :- blue(A).
inv2(A) :- size(A,B),small(B).
    \end{verbatim}

    \begin{center}
        \textbf{Aleph}
    \end{center}
\begin{verbatim}
zendo(27).  zendo(10).

zendo(A):-piece(A,B),coord2(B,C),medium(C),
size(B,D),small(D),strange(B).

zendo(E):-piece(E,F),contact(F,G),
size(G,H),small(H),strange(F).

zendo(I):-piece(I,J),coord1(J,K),size(J,K),
small(K).

zendo(L):-piece(L,M),coord2(M,N),
not small(N), size(M,O),small(O),rhs(M).
\end{verbatim}
    \item \textbf{Z6}: \underline{Largest is red} Is the largest cone in the scene red
   \begin{center}
        \textbf{\name{} \& \nameb{}}
    \end{center}
    \begin{verbatim}
zendo(A) :- piece(A,B),not inv1(B,A).

inv1(A,B) :- piece(B,C),size(C,D),
not inv2(D,A).

inv2(A,B) :- size(B,C),red(B),leq(A,C).
    \end{verbatim}

    \begin{center}
        \textbf{Aleph}
    \end{center}
\begin{verbatim}
zendo(29). zendo(28).  zendo(25).  
zendo(18). zendo(13).  zendo(11).  
zendo(9).  zendo(7). zendo(3). 

zendo(A) :- piece(A,B),coord2(B,C),
medium(C),rhs(B),blue(B).    

zendo(D) :- piece(D,E),coord2(E,F),large(F),
size(E,G),medium(G),upright(E).  

zendo(H) :- piece(H,I),contact(I,J),
coord2(J,K), upright(J),coord2(I,K),
not size(I,K).  

zendo(L) :- piece(L,M),coord1(M,N),lhs(M),
piece(L,O),coord2(O,N),strange(O).   

zendo(P) :- piece(P,Q),coord2(Q,R),small(R),
size(Q,S),medium(S),lhs(Q).  

zendo(T) :- piece(T,U),coord2(U,V),large(V),
strange(U),piece(T,W),contact(W,X).
\end{verbatim}
\end{itemize}
\subsection{Sets}
\begin{itemize}
  \item \textbf{S1}: \underline{Subset} The set B is a subset of the set A.
  \begin{center}
        \textbf{\name{} \& \nameb{}}
    \end{center}
    \begin{verbatim}
    subset(A,B) :- not inv1(A,B).
      inv1(A,B) :- member(C,B),not member(C,A).
    \end{verbatim}

    \begin{center}
        \textbf{Aleph}
    \end{center}
\begin{verbatim}
subset([x,s,y,z],[z,s,x]).
subset([x,s,y,z],[y,x]).
subset([x,y,z],[]). 
subset([x,y,z],[z,y]). 
subset([x,y,z],[x,z]). 
subset([x,y,z],[x,y]).
\end{verbatim}
    \item \textbf{S2}: \underline{Distinct} The set $A$ is distinct from the set $B$ 
   \begin{center}
        \textbf{\name{} \& \nameb{}}
    \end{center}
    \begin{verbatim}
distinct(A,B) :- not inv1(A,B).
    inv1(A,B) :- member(C,A),member(C,B).
    \end{verbatim}

    \begin{center}
        \textbf{Aleph}
    \end{center}
\begin{verbatim}
distinct([1,6,3,9,13,14,15,2],[10,17,11,8]).
distinct([x,s,y,z],[w,r,k,e]).
distinct([x,y,z],[w,r,e]).
\end{verbatim}
   \item \textbf{S3}: \underline{Set Difference} The set $Z$ is the difference of the sets $X$ and $Y$.
  \begin{center}
        \textbf{\name{} \& \nameb{}}
    \end{center}
    \begin{verbatim}
setdiff(A,B,C):- not inv1(C,A),not inv1(B,C).

inv1(A,B):- member(C,B),member(D,A),
not member(D,B),member(C,A).
    \end{verbatim}

    \begin{center}
        \textbf{Aleph}
    \end{center}
\begin{verbatim}
setdiff([x,y,z],[x,y,z],[]).
setdiff([r,w,y,k],[x,r,y,z,w],[k]).
setdiff([x,r,y,z,w],[r,w,y,k],[x,z]).
setdiff(A,B,A).
\end{verbatim}
  \item \textbf{S4}: \underline{Sym. Difference} The set $Z$ is the symmetric difference of $X$ and $Y$.
  \begin{center}
        \textbf{\name{} \& \nameb{}}
    \end{center}
    \begin{verbatim}
symmetricdiff(A,B,C):- my_union(A,C,B),
not inv1(B,C,A).

inv1(A,B,C):- member(D,A),member(D,B),
member(D,C).
    \end{verbatim}

    \begin{center}
        \textbf{Aleph}
    \end{center}
\begin{verbatim}
symmetricdiff([r,x,w,y],[x,y],[r,w]).
symmetricdiff([x,y],[r,x,w,y],[r,w]).
symmetricdiff([x,y,s,k],[x,w,y,r],[s,k,w,r]).
symmetricdiff([x,y],[z,w],[x,y,z,w]).
\end{verbatim}
    \item \textbf{S5}: \underline{Subset Decom.} $Y$ is a decomposition of $X$ into subsets of $X$.
  \begin{center}
        \textbf{\name{} \& \nameb{}}
    \end{center}
    \begin{verbatim}
subsetdecom(A,B):- not inv1(A,B).

inv1(A,B):- member(C,A),
missing_from_bucket(B,C).

inv1(A,B):- member_2(C,B),member(D,C), 
not member(D,A).
    \end{verbatim}
  \begin{center}
        \textbf{\popper{}}
    \end{center}
    \begin{verbatim}
subsetdecom(A,B):- member_2(A,B).
    \end{verbatim}

    \begin{center}
        \textbf{Aleph}
    \end{center}
\begin{verbatim}
subsetdecom([x,y,z,w,r,k,s],
[[w,x,z],[y,r],[k,s]]).

subsetdecom([x,y,z,w],[[w,x,z],[y]]).
subsetdecom([x,y,z,w],[[x,y,z,w]]).
subsetdecom([x,y,z,w],[[x,y],[z,w]]).
\end{verbatim}
   
        \item  \textbf{S6}: \underline{Mutual distinct} The sets contained in the set $X$ are mutual distinct
  \begin{center}
        \textbf{\name{} \& \nameb{}}
    \end{center}
    \begin{verbatim}
 mutualdistinct(A) :- not inv1(A).
inv1(A) :- member(B,A),member(C,A),not inv2(B,C).
inv2(A,B) :- eq(A,B).
inv2(A,B) :- distinct(A,B).
    \end{verbatim}

    \begin{center}
        \textbf{Aleph}
    \end{center}
\begin{verbatim}
mutualdistinct([]).
mutualdistinct(A):-member(B,A),member(C,A),
distinct(C,B), member(D,A),distinct(D,C),
distinct(D,B).
\end{verbatim}

\end{itemize}

\subsection{Graph Problems}
\begin{itemize}
     \item \textbf{G1}: \underline{Independent Set} Is $B$ an independent set of $A$.
     \begin{center}
        \textbf{\name{} \& \nameb{}}
    \end{center}
    \begin{verbatim}
 independent(A,B):- not inv1(B,A).
 
 inv1(A,B):- member(C,A),edge(B,D,C),
 member(D,A).
    \end{verbatim}
    \begin{center}
        \textbf{Aleph}
    \end{center}
\begin{verbatim}
independent(d,[5,7,11,12]).
independent(c,[2,4,7,9]).
independent(b,[2,4,7,9]).
independent(a,[3,4,5,11]).
\end{verbatim}

    \item \textbf{G2}: \underline{Star Graph} Does $A$ contain a node with an edge to all other nodes of $A$ (a star).
     \begin{center}
        \textbf{\name{} \& \nameb{}}
    \end{center}
    \begin{verbatim}
starg(A):- node(A,B),not inv1(B,A).

inv1(A,B):- node(B,C),not edge(B,A,C),
not eq(A,C).
    \end{verbatim}
     \begin{center}
        \textbf{\popper}
    \end{center}
    \begin{verbatim}
    starg(A):- edge(A,B,C),edge(A,C,B).

    \end{verbatim}

    \begin{center}
        \textbf{Aleph}
    \end{center}
\begin{verbatim}
starg(b).
starg(a).
\end{verbatim}
     \item \textbf{G3}: \underline{Unconnected} The node $A$ does not have a path to $B$ in the graph.
     \begin{center}
        \textbf{\name{} \& \nameb{}}
    \end{center}
    \begin{verbatim}
unconnected(A,B) :- not inv1(B,A).
inv1(A,B):- edge(B,A).
inv1(A,B):- edge(B,C),inv1(C,A).
    \end{verbatim}
    \begin{center}
        \textbf{Aleph}
    \end{center}
\begin{verbatim}
unconnected(8,9).  unconnected(9,8). 
unconnected(7,8).  unconnected(6,8). 
unconnected(5,8).  unconnected(4,8).  
unconnected(3,8).  unconnected(2,8). 
unconnected(1,8).  unconnected(8,7).  
unconnected(8,6).  unconnected(8,5).  
unconnected(8,4).  unconnected(8,3).  
unconnected(8,2).  unconnected(8,1).
\end{verbatim}
       \item \textbf{G4}: \underline{Proper Subgraph} Is $B$ is a  proper subgraph of $A$.
        \begin{center}
        \textbf{\name{} \& \nameb{}}
    \end{center}
    \begin{verbatim}
propersubgraph(A,B) :- not inv1(A,B),
node(A,C), not node(B,C).
inv1(A,B) :- edge(B,D,C), not edge(A,D,C).
    \end{verbatim}
        \begin{center}
        \textbf{\popper}
    \end{center}
    \begin{verbatim}
    propersubgraph(A,B):- edge(B,C,D),
    not edge(A,D,C).
    \end{verbatim}

    \begin{center}
        \textbf{Aleph}
    \end{center}
\begin{verbatim}
propersubgraph(A,B):- edge(B,C,D),edge(A,C,E),
not node(B,E).
\end{verbatim}
    \item \textbf{G5}: \underline{red-green neighbor}  Every red node of $A$ has a green neighbor.
\begin{center}
        \textbf{\name{} \& \nameb{}}
    \end{center}
    \begin{verbatim}
redGreenNeighbor(A) :- not inv1(A).
inv1(A) :- red(A,B), not inv2(B,A).
inv2(A,B) :- green(B,C),edge(B,A,C).
    \end{verbatim}
    \begin{center}
        \textbf{Aleph}
    \end{center}
\begin{verbatim}
redGreenNeighbor(f). redGreenNeighbor(e).
redGreenNeighbor(d). redGreenNeighbor(c). 
redGreenNeighbor(b). redGreenNeighbor(a).
\end{verbatim}
        
     \item  \textbf{G6}: \underline{max node weight} $G$ is a graph with  weighted nodes and $U$ has the maximum weight.
\begin{center}
        \textbf{\name{} \& \nameb{}}
    \end{center}
    \begin{verbatim}
maxweightnode(A,B):- weight(A,B,C), 
not inv1(A,C).  

inv1(A,B):- node(A,C),not inv2(A,B,C).

inv2(A,B,C):- weight(A,C,D),leq(D,B).
    \end{verbatim}
    \begin{center}
        \textbf{Aleph}
    \end{center}
\begin{verbatim}
maxweightnode(b,6).
maxweightnode(a,8).
\end{verbatim}
   \item \textbf{G7}: \underline{dominating set} is $B$ a dominating set of $A$.
                \begin{center}
        \textbf{\name{} \& \nameb{}}
    \end{center}
    \begin{verbatim}
dominating(A,B):- not inv1(A,B).

inv2(A,B,C):- edge(A,B,D), member(C,D).

inv1(A,B):- node(A,C), not member(C,B), 
not inv2(A,C,B).
    \end{verbatim}
    \begin{center}
        \textbf{Aleph}
    \end{center}
\begin{verbatim}
dominating(b,[1,2,3,4]).
dominating(a,[1,2,10]).
\end{verbatim}
    \item \textbf{G8}: \underline{maximal independent set}  $B$ is a maximal independent set of $A$.
\begin{center}
        \textbf{\name{} \& \nameb{}}
    \end{center}
    \begin{verbatim}
max_independent(A,B):- not inv1(B,A).

inv2(A,B,C):- member(D,B),edge(C,A,D).

inv1(A,B):- node(B,C),not member(C,A), 
not inv2(C,A,B).
    \end{verbatim}
    \begin{center}
        \textbf{Aleph}
    \end{center}
\begin{verbatim}
max_independent(d,[1,2,6,10,14]).
max_independent(c,[6,7,9,13,14]).
max_independent(b,[5,6,7,8,10,12,13,14]).
max_independent(a,[2,4,7,9]).
\end{verbatim}
    \end{itemize}

\section{Details: Table~\ref{tab:q2times} Construction}
\label{app:tables}
 We run 125 trials and set the bias as follows: max variables 4, max rules 4, max body literals 4. The only exception to the bias settings is \textit{unconnected}. Constructing graphs that capture \textit{unconnected}  without capturing a simpler property is non-trivial. Due to the non-deterministic behaviour of \popper{}'s search mechanism, we commuted the likelihood that the mean (and median) of the trails of \name{} and \nameb{} differ.

 \begin{table}[ht]
\centering
\footnotesize
\begin{tabular}{@{}l|cccc}
\textbf{Task} & 
\textbf{\ale{}} & 
\textbf{\metagolsn{}} & 
\textbf{\name{}} &  
\textbf{\popper{}}\\
\midrule
graph1      & 1  $\pm$ 0  & 1  $\pm$ 0      & 100  $\pm$ 20  &  1  $\pm$ 0  \\
graph2      & 1  $\pm$ 0  & 270 $\pm$ 77    & 300   $\pm$ 0   &  0  $\pm$ 0  \\
graph3      & 3  $\pm$ 0  & 213  $\pm$ 113  & 112   $\pm$ 0   &  0  $\pm$ 0  \\
graph4      & 3  $\pm$ 0  & 180  $\pm$ 126  & 213   $\pm$ 0   &  0  $\pm$ 0  \\
graph5      & 4  $\pm$ 0  & 300  $\pm$ 0    & 300   $\pm$ 0   &  0  $\pm$ 0  \\
\midrule
imdb1   & 135  $\pm$ 64   & 300  $\pm$ 0   &   1   $\pm$ 0   &    1  $\pm$ 0  \\
imdb2   & 300  $\pm$ 0    & 300  $\pm$ 0   &   2   $\pm$ 0   &    2  $\pm$ 0  \\
imdb3   & 300  $\pm$ 0    & 300  $\pm$ 0   & 300   $\pm$ 0   &  300  $\pm$ 0  \\
\midrule

krk1    & 0  $\pm$ 0  & 300  $\pm$  0    &  44   $\pm$ 17  &  62  $\pm$ 30 \\ 
krk2    & 8  $\pm$ 3  & 300  $\pm$  0    & 300   $\pm$  0  & 300  $\pm$ 0  \\
krk3    & 1  $\pm$ 0  & 279  $\pm$  36   & 276   $\pm$ 62  & 300  $\pm$ 0   \\ 

\midrule
contains    &  55   $\pm$ 8  &   0  $\pm$ 0    & 300  $\pm$ 0  &  25  $\pm$ 3  \\
dropk       &   7   $\pm$ 3  &   0  $\pm$ 0    & 300  $\pm$ 0  &   2  $\pm$ 1  \\
droplast    & 300   $\pm$ 0  &   0  $\pm$ 0    & 300  $\pm$ 0  &   2  $\pm$ 0  \\ 
evens       &   2   $\pm$ 1  &   0  $\pm$ 0    & 300  $\pm$ 0  &   3  $\pm$ 0 \\
finddup     &   1   $\pm$ 0  & 183  $\pm$ 123  & 300  $\pm$ 0  &   9  $\pm$ 1  \\
last        &   2   $\pm$ 0  & 134  $\pm$ 119  & 300  $\pm$ 0  &   2  $\pm$ 1  \\
len         &   2   $\pm$ 0  & 206  $\pm$ 100  & 300  $\pm$ 0  &   4  $\pm$ 0  \\
reverse     &   0   $\pm$ 0  &   0  $\pm$ 0    & 300  $\pm$ 0  & 300  $\pm$ 0  \\
sorted      &   1   $\pm$ 1  &   0  $\pm$ 0    & 300  $\pm$ 0  &  21  $\pm$ 8  \\
sumlist     &   0   $\pm$ 0  & 210  $\pm$ 117  & 300  $\pm$ 0  & 300  $\pm$ 0  \\
\midrule
attrition          &   2  $\pm$ 0  &   0  $\pm$ 0   & 300  $\pm$ 0   & 300  $\pm$ 0 \\
buttons            &  75  $\pm$ 0  & 300  $\pm$ 0   &   8  $\pm$ 0   &   6  $\pm$ 0  \\
buttons-g          &   0  $\pm$ 0  &   0  $\pm$ 0   &   3  $\pm$ 0   &   2  $\pm$ 0 \\
coins              & 300  $\pm$ 2  &   0  $\pm$ 0   & 148  $\pm$ 6   & 116  $\pm$ 4  \\
coins-g            &   0  $\pm$ 3  &   0  $\pm$ 0   &   3  $\pm$ 0   &   2  $\pm$ 0 \\
centipede-g        &   1  $\pm$ 0  &   0  $\pm$ 0   &  47  $\pm$ 1   &  44  $\pm$ 0 \\
md                 &   2  $\pm$ 0  & 300  $\pm$ 0   &   4  $\pm$ 0   &   3  $\pm$ 0 \\
rps                &   4  $\pm$ 0  &   0  $\pm$ 0   & 197  $\pm$ 30  &  75  $\pm$ 6 \\
\midrule
trains1     &    3  $\pm$ 1  &    300  $\pm$ 0   &   300  $\pm$ 0   &    5  $\pm$ 0 \\
trains2     &    2  $\pm$ 0  &    300  $\pm$ 0   &   300  $\pm$ 1   &    5  $\pm$ 1 \\
trains3     &    7  $\pm$ 2  &    300  $\pm$ 0   &   300  $\pm$ 0   &   30  $\pm$ 0 \\
trains4     &   20  $\pm$ 8  &    300  $\pm$ 0   &   300  $\pm$ 30  &   25  $\pm$ 1 \\
\midrule
zendo1     &    4  $\pm$ 3  &   300  $\pm$ 0   &   300  $\pm$ 0   &    44  $\pm$ 20 \\
zendo2     &    6  $\pm$ 2  &   300  $\pm$ 0   &   300  $\pm$ 1   &    89  $\pm$ 5 \\
zendo3     &   10  $\pm$ 2  &   300  $\pm$ 0   &   300  $\pm$ 0   &    82  $\pm$ 8 \\
zendo4     &   10  $\pm$ 4  &   300  $\pm$ 0   &   300  $\pm$ 30  &    56  $\pm$ 8 \\
\end{tabular}
\caption{
Learning times average (10 runs). We round times to the nearest second and provide 99\% confidence intervals .
}
\label{tab:NegAlephMetalearn}
\end{table}

 \begin{table}[ht]
\centering
\footnotesize
\begin{tabular}{@{}l|cccc}
\textbf{Task} & 
\textbf{\ale{}} & 
\textbf{\metagolsn{}} & 
\textbf{\name{}} &  
\textbf{\popper{}}\\
\midrule
graph1      & 100  $\pm$ 0  & 100  $\pm$ 0    & 100  $\pm$ 20  &   100  $\pm$ 0  \\
graph2      & 93  $\pm$ 2   &  54 $\pm$  0    & 0   $\pm$ 0   &    100  $\pm$ 0  \\
graph3      & 100  $\pm$ 0  &  65  $\pm$ 0    & 100   $\pm$ 0   &  100 $\pm$ 0  \\
graph4      & 99  $\pm$ 2   &  70  $\pm$ 0    &  99   $\pm$ 0   &  100 $\pm$ 0  \\
graph5      & 99  $\pm$ 1   &  0  $\pm$ 0     &   0   $\pm$ 0   &  100  $\pm$ 0  \\
\midrule
imdb1   & 95  $\pm$ 25    & 0  $\pm$ 0   &   100   $\pm$ 0   &    100  $\pm$ 0  \\
imdb2   & 50  $\pm$ 0    & 0  $\pm$ 0   &   100   $\pm$ 0   &     100  $\pm$ 0  \\
imdb3   & 50  $\pm$ 0    & 0  $\pm$ 0   &     0   $\pm$ 0   &       0  $\pm$ 0  \\
\midrule

krk1    & 98  $\pm$ 3  & 0  $\pm$   0    &  98   $\pm$ 0  &  100  $\pm$ 0 \\ 
krk2    & 95  $\pm$ 2  & 0  $\pm$   0    &   0   $\pm$  0  &   0  $\pm$ 0  \\
krk3    & 95  $\pm$ 7  & 50  $\pm$  0    &   5   $\pm$ 13  &   0  $\pm$ 0   \\ 

\midrule
contains    &   51   $\pm$ 1  &   0  $\pm$ 0    & 0  $\pm$ 0  &  100  $\pm$ 3  \\
dropk       &   50   $\pm$ 0  &   0  $\pm$ 0    & 0  $\pm$ 0  &  100  $\pm$ 1  \\
droplast    &    0  $\pm$  0  &   0  $\pm$ 0    & 0  $\pm$ 0  &  100  $\pm$ 0  \\ 
evens       &   66   $\pm$ 17  &   0  $\pm$ 0   & 0  $\pm$ 0  &  100  $\pm$ 0 \\
finddup     &   50   $\pm$ 0  &   0  $\pm$ 0    & 0  $\pm$ 0  &   99  $\pm$ 1  \\
last        &   50   $\pm$ 0  &   0  $\pm$ 0    & 0  $\pm$ 0  &  100  $\pm$ 1  \\
len         &   50   $\pm$ 0  &   0  $\pm$ 0    & 0  $\pm$ 0  &  100  $\pm$ 0  \\
reverse     &   50   $\pm$ 0  &   0  $\pm$ 0    & 0  $\pm$ 0  &    0  $\pm$ 0  \\
sorted      &   74  $\pm$  7  &   0  $\pm$ 0    & 0  $\pm$ 0  &   98  $\pm$ 5  \\
sumlist     &   50   $\pm$ 0  &   0  $\pm$ 0    & 0  $\pm$ 0  &  100  $\pm$ 0  \\
\midrule
attrition        &  93  $\pm$ 0  &  93  $\pm$ 0   &   0  $\pm$ 0   &   0  $\pm$ 0 \\
buttons          &  87  $\pm$ 0  &   0  $\pm$ 0   & 100  $\pm$ 0   & 100  $\pm$ 0  \\
buttons-g        & 100  $\pm$ 0  &  50  $\pm$ 0   &  98  $\pm$ 1   & 100  $\pm$ 0 \\
coins            &   0  $\pm$ 0  &  17  $\pm$ 0   & 100  $\pm$ 0   & 100  $\pm$ 4  \\
coins-g          &  50  $\pm$ 0  &  50  $\pm$ 0   & 100  $\pm$ 0   & 100  $\pm$ 0 \\
centipede-g      &  98  $\pm$ 0  &   5  $\pm$ 0   & 100  $\pm$ 0   & 100  $\pm$ 0 \\
md               &  94  $\pm$ 0  &   0  $\pm$ 0   & 100  $\pm$ 0   & 100  $\pm$ 0 \\
rps              & 100  $\pm$ 0  &  19  $\pm$ 0   & 100  $\pm$ 0  &  100  $\pm$ 6 \\
\midrule
trains1     &   100  $\pm$ 0  &    0  $\pm$ 0   &   0  $\pm$ 0   &   100  $\pm$ 0 \\
trains2     &   100  $\pm$ 0  &    0  $\pm$ 0   &   0  $\pm$ 0   &    95  $\pm$ 0 \\
trains3     &   100  $\pm$ 0  &    0  $\pm$ 0   &   0  $\pm$ 0   &   100  $\pm$ 0 \\
trains4     &   100  $\pm$ 0  &    0  $\pm$ 0   &   0  $\pm$ 0  &    100  $\pm$ 0 \\
\midrule
zendo1     &   98  $\pm$ 0  &   0  $\pm$ 0   &   0  $\pm$ 0   &    98  $\pm$ 0 \\
zendo2     &   99  $\pm$ 0  &   0  $\pm$ 0   &   0  $\pm$ 0   &   100  $\pm$ 0 \\
zendo3     &  100  $\pm$ 0  &   0  $\pm$ 0   &   0  $\pm$ 0   &   100  $\pm$ 0 \\
zendo4     &   99  $\pm$ 1  &   0  $\pm$ 0   &   0  $\pm$ 0  &    100  $\pm$ 0 \\
\end{tabular}
\caption{
predictive accuracy average (10 runs). We round times to the nearest second and provide 99\% confidence intervals.
}
\label{tab:NegAlephMetaacc}
\end{table}

\begin{table}[ht]
\centering
\footnotesize
\begin{tabular}{@{}l|cc@{}}
\textbf{Task} & 
  \textbf{Sound constraints} & 
  \textbf{Unsound constraints}\\
\midrule
gen &  \textbf{100  $\pm$ 0 }&  \textbf{0  $\pm$ 0 }   \\
spec &   \textbf{100  $\pm$ 0 }  &  \textbf{0  $\pm$ 0 }  \\
div. list & {100  $\pm$ 0 } &  100  $\pm$ 0   \\
\midrule
1 even      & {100  $\pm$ 0 }  &  100 $\pm$ 0    \\
leapyear          & {100  $\pm$ 0 }  &  100 $\pm$ 0  \\
\hline

not up  & {100  $\pm$ 0 } &  100  $\pm$ 0   \\ 
all red     & {100  $\pm$ 0 } &  100  $\pm$ 0     \\
all eq. size    & {100  $\pm$ 0 }  &  100  $\pm$ 0    \\ 
 one bl.   & \textbf{100  $\pm$ 0 } &  \textbf{55  $\pm$ 0} \\
bl. or sm. & {100  $\pm$ 0 } &   100  $\pm$ 0   \\ 
largest red    & {100  $\pm$ 0 } &   100 $\pm$ 0 \\ 

\hline
ind. set  & {100  $\pm$ 0 } &  100  $\pm$ 0   \\
star        & \textbf{100  $\pm$ 0 }  & \textbf{24  $\pm$ 0 }   \\
unconn.      & \textbf{100  $\pm$ 0 } & \textbf{92 $\pm$ 0  }   \\ 
p. subgraph  & {100  $\pm$ 0 }  &  100  $\pm$ 0   \\
rgn   & {100  $\pm$ 0 }  &  100 $\pm$ 0   \\
max weight  & {100  $\pm$ 0 } &  100  $\pm$ 0     \\
dom. set   &  \textbf{100  $\pm$ 0 } &  \textbf{0  $\pm$ 0 } \\
max. ind. set & \textbf{100  $\pm$ 0 }  &  \textbf{88 $\pm$ 0}  \\
\hline

subset          & {100  $\pm$ 0 }  &  100  $\pm$ 0  \\
distinct         & {100  $\pm$ 0 } & 100  $\pm$ 0   \\
set diff.   & \textbf{92  $\pm$ 0} &  \textbf{50 $\pm$ 0}  \\
sym. diff.  & {100  $\pm$ 0 } & 100 $\pm$ 0  \\
decom.      &  {100  $\pm$ 0 }  & 100 $\pm$ 0  \\
m-distinct  &  {100  $\pm$ 0 } & 100 $\pm$ 0 \\

\end{tabular}
\caption{
Mean predictive accuracies (10 runs). We round times to the nearest second and provide 99\% confidence intervals. Ordered by domain and optimal solution size.
}
\label{tab:PredictiveAcc2}
\end{table}

\section{Generate Phase ASP: Polar}
\label{sec:ASPencoding}
Below we provide the ASP code restricting the generation of normal logic programs to polar programs only. We will explain how each part relates to the definition provided in Section~\ref{sec:setting} within the main body of the paper. Some of the names are abbreviated for space reasons. 
\begin{align}
        pol(0..1).&\label{aln:eq1}\\
polEq(Ph,1,Pb)\mbox{:-}& h\_lit(C,Ph,\_,\_),\label{aln:eq2}\\
                       & b\_lit(C,@n\_bk\_c(Pb),\_,\_) \nonumber\\
                       & invented(Pb,\_). \nonumber\\
polEq(Ph,0,Pb)\mbox{:-}& h\_lit(C,Ph,\_,\_),\label{aln:eq3} \\
                       & b\_lit(C,Pb,\_,\_), \nonumber\\
                       & invented(Pb,\_).  \nonumber\\
polEq(P,0,P)\mbox{:-}  & h\_lit(C,P,\_,\_),\label{aln:eq4}\\
                       & not\ b\_lit(\_,P,\_,\_), \nonumber\\
                       & not\ b\_lit(\_,@n\_bk\_c(P),\_,\_). \nonumber\\
polEq(Ph,X\ \mbox{\textasciicircum}\ Y,Pb2)\mbox{:-} & h\_pred(Ph,\_),\label{aln:eq5}\\
                     & polEq(Ph,X,Pb), \nonumber\\
                     & polEq(Pb,Y,Pb2). \nonumber
\end{align}
\begin{align}
negcls(C)\mbox{:-}& h\_pred(Ph,\_),\label{aln:eq6}\\
               & polEq(Ph,1,Pb),\nonumber\\
               & h\_lit(C,Pb,\_,\_).\nonumber\\
poscls(C)\mbox{:-}& h\_pred(Ph,\_),\label{aln:eq7}\\
               & polEq(Ph,0,Pb),\nonumber\\
               & h\_lit(C,Pb,\_,\_).\nonumber\\
negcnt(X)\mbox{:-}& \#count\{C: negcls(C)\}=X.\label{aln:eq8}\\
poscnt(X)\mbox{:-}& \#count\{C: poscls(C)\}=X.\label{aln:eq9}\\
\mbox{:-}& clause(C),\label{aln:eq10}\\ 
         & poscls(C),\nonumber\\ 
         & negcls(C).\nonumber 
\end{align}

In the above ASP code, $h\_lit/4$ and $b\_lit/4$ refer to head and body literals within a given normal program, respectively. The predicates $h\_pred/2$ and $invented/2$ denote user-defined head predicate symbols and invented symbols, respectively. And $@n\_bk\_c$ is a function used to generate a negated version of the input predicate symbol. 

Equation~\ref{aln:eq1} defines that there are only two polarities. The top set (Definition~\ref{def:topset}) is defined using Equations~\ref{aln:eq3}~\&~\ref{aln:eq4}. The constraints of Definition~\ref{def:polarity} are implemented using Equations~\ref{aln:eq2},~\ref{aln:eq3},~\&~\ref{aln:eq5}. Equations~\ref{aln:eq6}~\&~\ref{aln:eq7} assign polarities to clauses of the program. Equations~\ref{aln:eq8}~\&~\ref{aln:eq9} are used for building constraints. And finally, Equation~\ref{aln:eq10} restricts the output model to polar programs, Definition~\ref{def:polarProg}.

\subsection{Encoding Correctness (Sketch)}
If a predicate $p$ occurs in the body of a clause with head symbol $q$ we say $q$ \textit{calls} $p$; what we refer to as the \textit{call relation}. Equations~\ref{aln:eq2},~\ref{aln:eq3},~\&~\ref{aln:eq4} Define the call relation for non-background knowledge predicates.  Equation~\ref{aln:eq5} defines the transitive closure of the call relation. As defined, these four equations also capture the notion of top symbol, however, the only top symbols which we need to consider are those defined as \textit{head predicates} (\textit{h\_pred/2}). All non-background predicates must have a path within the transitive closure of the call relation to head predicates and head predicates never occur negated. Invented predicates are not defined as \textit{head predicates}, but rather auxiliary head predicates.  

The predicate \textit{polEq} takes an addition argument denoting whether the third argument occurred negatively (1) or positively (0) in the body of a clause whose head symbol is the first argument. We refer to the second argument as the \textit{parity}. In Equation~\ref{aln:eq5}, when computing the transitive closure we \textit{Xor} the parities. Thus, we keep track of how many negations occur on paths between a given symbol and the head predicates; this is clearly indicated in Equations~\ref{aln:eq6}~\&~\ref{aln:eq7} which assign clauses to being positive or negative based on the parity between a head predicate and the head symbol of the clause. 

If a program is non-polar than some clauses will be assigned to both the positive and negative set. Such assignment occurs when the transitive closure of the call relation computes multiple parities for the paths between the head symbol of a clause and the head predicates. For example, if negated recursion occurs, or if a predicate symbol shows up positively and negatively in the same clause. Equations~\ref{aln:eq10} checks if any clause is assigned to both the positive and negative set. 

\section{Example Polar Constraints}

Below we provide polar generalisation and specialisation constraints generated when running \name{} on the graph task G7, \textit{dominating set}. We provide the program and the generated constraint.  

\subsection{Generalisation Constraint}
\begin{itemize}
    \item \textbf{Hypothesis}
    \begin{verbatim}
dominating(A,B):- member(C,B),\+ inv1(A,B,C).
inv2(A,B,C):- edge(A,D,C),member(D,B).
inv1(A,B,C):- \+ inv2(A,B,C),node(A,C). 
    \end{verbatim}
\item \textbf{Constraint}
\begin{verbatim}
:-
head_literal(R2,inv1,3,(R2VA, R2VB, R2VC)), 
body_literal(R2,not_inv2,3,(R2VA, R2VB, R2VC)), 
body_literal(R2,node,2,(R2VA, R2VC)), 
R2VA == 0, R2VB == 1, R2VC == 2,
head_literal(R0,dominating,2,(R0VA, R0VB)), 
body_literal(R0,member,2,(R0VC, R0VB)), 
body_literal(R0,ournot_inv1,3,(R0VA,R0VB,R0VC)), 
R0VA == 0, R0VB == 1, R0VC == 2,
head_literal(R1,inv2,2,(R1VA, R1VB, R1VC)), 
body_literal(R0,edge,2,(R1VA, R1VD,R1VC)), 
body_literal(R1,member,3,(R1VD,R1VB)), 
R1VA == 0, R1VB == 1, R1VC == 2, R1VD == 3,
body_size(R0,2), 
body_size(R1,2),
posclause(R0),
posclause(R1), 
negclause(R2), 
negcount(1), 
R2 < R1, R0 < R1, R0 < R2.
\end{verbatim}
\end{itemize}
\subsection{Specialisation Constraint}

\begin{itemize}
 \item \textbf{Hypothesis}
    \begin{verbatim}
dominating(A,B):- member(C,B),\+ inv1(A,B,C).
inv2(A,B,C):- edge(A,C,D),member(D,B).
inv1(A,B,C):- node(A,C),\+ inv2(A,B,C). 
    \end{verbatim}
\item \textbf{Constraint}
\begin{verbatim}
:-
head_literal(R2,inv1,3,(R2VA, R2VB, R2VC)), 
body_literal(R2,not_inv2,3,(R2VA, R2VB, R2VC)), 
body_literal(R2,node,2,(R2VA, R2VC)), 
R2VA == 0, R2VB == 1, R2VC == 2,
head_literal(R0,dominating,2,(R0VA, R0VB)), 
body_literal(R0,member,2,(R0VC, R0VB)), 
body_literal(R0,ournot_inv1,3,(R0VA,R0VB,R0VC)), 
R0VA == 0, R0VB == 1, R0VC == 2,
head_literal(R1,inv2,2,(R1VA, R1VB, R1VC)), 
body_literal(R0,edge,2,(R1VA, R1VC,R1VD)), 
body_literal(R1,member,3,(R1VD,R1VB)), 
R1VA == 0, R1VB == 1, R1VC == 2, R1VD == 3,
body_size(R0,2), 
body_size(R1,2),
posclause(R0),
posclause(R1), 
negclause(R2), 
poscount(2),
R2 < R1, R0 < R1, R0 < R2.
\end{verbatim}
\end{itemize}

\end{document}